\documentclass[11pt]{article}

\usepackage[plainpages=false,pdfpagelabels,colorlinks=true,linkcolor=myred,urlcolor=mypurple,citecolor=myblue]{hyperref}
\usepackage{tptstyle}

\usepackage{fullpage}
\usepackage[font={footnotesize,it}]{caption}

\begin{document}

\title{CINDES: Classification induced neural density estimator and simulator}
\author{Dehao Dai, Jianqing Fan, Yihong Gu, and Debarghya Mukherjee}
\date{}

\maketitle

\begin{abstract}
Neural network–based methods for (un)conditional density estimation have recently gained substantial attention, as various neural density estimators have outperformed classical approaches in real-data experiments. Despite these empirical successes, implementation can be challenging due to the need to ensure non-negativity and unit-mass constraints, and theoretical understanding remains limited. In particular, it is unclear whether such estimators can adaptively achieve faster convergence rates when the underlying density exhibits a low-dimensional structure. This paper addresses these gaps by proposing a structure-agnostic neural density estimator that is (i) straightforward to implement and (ii) provably adaptive, attaining faster rates when the true density possesses a low-dimensional structure. Another key contribution of our work is to show that the proposed estimator integrates naturally into sampling pipelines, most notably score-based diffusion models, where it achieves provably faster convergence when the underlying density is structured. We validate its performance through extensive simulations and a real-data application.
\end{abstract} 

\section{Introduction}
Density estimation is a fundamental and now classical problem in Statistics and Machine Learning (ML), which has been widely applied in astronomy, climatology, economics, medicine, genetics, physiology, and other fields.  
Starting from the middle of the 20th century, numerous methods have been developed (e.g., Kernel-based \citep{parzen1962estimation,nadaraya1964estimating,watson1963estimation}, series-based \citep{efroimovich1982estimation}, etc., see \cite{silverman2018density} for a comprehensive discussion).  
While traditional methods for density estimation are well-analyzed and straightforward to implement, they frequently encounter significant challenges due to the curse of dimensionality in high-dimensional observations. Therefore, to achieve faster convergence rates, it is crucial to employ methods that can effectively exploit any underlying low-dimensional structure in the target density.   This is where the neural network techniques come into play. 
Recently, in the context of nonparametric regression, several researchers \citep{kohler2021rate, schmidt2020nonparametric,fan2024noise,fan2024factor} have demonstrated that deep neural networks (DNNs) can efficiently exploit low-dimensional or compositional structures in the underlying regression function, resulting in estimators that achieve faster convergence rates depending on the nature of the structure. This paper demonstrates that similar properties hold in the context of density estimation; if a high-dimensional (conditional or unconditional) density function exhibits an underlying low-dimensional or compositional structure, then DNN-based density estimators can indeed achieve faster rates of convergence by leveraging this structure.

Density estimation using DNNs has recently gained significant popularity, resulting in the development of various approaches. Broadly, these methods can be categorized into two classes: (i) \textit{explicit density estimation}, and (ii) \textit{implicit density estimation} as in generative AI. 
Explicit density estimation directly parameterizes the conditional or unconditional density function using a neural network, and the model parameters are learned from the data. 
Notable examples include minimizing a squared error loss for density estimation using DNNs \citep{bos2023supervised}, and learning conditional densities through variational autoencoders (VAEs) \citep{kingma2013auto, rezende2014stochastic, higgins2017beta, tolstikhin2017wasserstein}. 
On the other hand, implicit modeling focuses on generating samples from the target distribution without explicitly specifying its density function. A variety of such methods have been proposed in the literature, including generative adversarial networks (GANs) \citep{goodfellow2014generative, arjovsky2017wasserstein, mescheder2018training, liang2021well, singer2018mathematics, tang2023minimax, stephanovitch2023wasserstein}, score-based generative models \citep{song2019generative, song2020score, benton2023nearly,chen2024overview, huang2024denoising}, and normalizing flows. The common goal of these methods is to learn a transformation from a known noise distribution (e.g., Gaussian or Uniform) to the target distribution, enabling sample generation by passing noise through the learned map.  These methods do not output the estimated density, but generate data from the estimated density.

While numerous density estimation methods based on deep neural networks have been proposed, it remains theoretically unclear whether such approaches can adapt to the unknown structural properties of the target density and achieve faster convergence rates. For instance, consider a $d$-dimensional random variable $X = (X_1, \dots, X_d)$ with a Markov factorization: $f(x) = \prod_{j=1}^{d-1} f_j(x_{j+1}, x_j) = \exp(\sum_j \log f_j(x_{j+1}, x_j))$. Although this is a high-dimensional density, its structure is governed by $(d-1)$ two-dimensional components, suggesting the possibility of avoiding the curse of dimensionality. However, traditional methods (e.g., kernel-based/series-based methods) fail to exploit such structure without explicit prior knowledge.
In contrast, recent work in nonparametric regression has shown that deep neural networks can adapt to unknown low-dimensional structures, such as hierarchical compositions (e.g., \cite{kohler2021rate, schmidt2020nonparametric, fan2024factor, bhattacharya2024deep}). 
Building on this insight, we recast the density estimation problem as a classification task and demonstrate that neural networks can achieve faster convergence rates for conditional or unconditional density estimation, provided the underlying density exhibits inherent low-dimensional structure. Our method is broadly applicable and can be used in both explicit and implicit density estimation.

\subsection{Our contribution}

In this paper, we propose a structure-agnostic density estimation procedure using deep neural networks for both conditional and unconditional density estimation. Our methodology is inspired by the probabilistic classification approach \citep{qin1998inferences,bickel2007discriminative,cheng2004semiparametric}, which was primarily developed for estimating density ratios, which is also known as the \textit{likelihood ratio trick} in simulation-based inference \citep{cranmer2020frontier}, and thus we refer to our method as \emph{Classification induced neural density estimator and
simulator (CINDES)}. 

We now briefly outline our methodology for estimating the conditional density function. Suppose we observe a dataset $\cD_n = \{(X_1, Y_1), \dots, (X_n, Y_n)\}$, where $X \in \cX \subseteq \mathbb{R}^{d_x}$ and $Y \in \cY \subseteq \mathbb{R}^{d_y}$. Assume that $X \sim \mu_{0, x}$ and $Y \mid X \sim p_0(y \mid X)$ for some unknown $(\mu_{0, x}, p_0)$ and we aim to estimate the conditional density $p_0$. 
Our proposed procedure consists of two key steps: 
\begin{enumerate}
\setlength\itemsep{-.2em}
    \item Generate a set of ``fake responses" $\{\tilde Y_1, \dots, \tilde Y_n\}$ uniformly from (some superset of) $\cY$.

    \item Construct a synthetic dataset $\tilde \cD_n = \{(X_1, \tilde Y_1), \dots, (X_n, \tilde Y_n)\}$, and estimate the Bayes classifier distinguishing samples from the original dataset $\cD_n$ and the synthetic dataset $\tilde \cD_n$.\footnote{Although it is possible to generate more than $n$ samples from $\tilde Y$, this would not provide any additional information about the conditional density of $Y$ given $X$; it would only reveal information about the marginal density of $\tilde Y$, which is already known. Therefore, generating $n$ samples is sufficient, and producing more would not yield any further benefit.} 
\end{enumerate}
Since each $\tilde Y_i$ is sampled independently of $(X, Y)$ and uniformly over a superset of $\cY$, the joint density of $(X, \tilde Y)$ is equal to $C\mu_{0, x}$ where $C^{-1} = \mathrm{Leb}(\cY)$, the Lebesgue measure of $\cY$.  
Moreover, by construction, the support of $(X, \tilde Y)$ covers that of $(X, Y)$. 
Therefore, the density ratio between $(X, Y)$ and $(X, \tilde Y)$ is proportional to $p_0(y \mid x)$, and consequently, estimating $p_0(y \mid x)$ reduces to estimating this density ratio. We then estimate this density ratio by employing probabilistic classification methods -- specifically, by learning a classifier to distinguish between samples from the distributions of $(X, Y)$ and $(X, \tilde Y)$. Further details are provided in Algorithm~\ref{algo:ede}. Our proposed procedure naturally extends to unconditional density estimation by setting $X = \varnothing$, i.e., effectively ignoring covariates during estimation.

The key advantage of our \emph{reduction} from density estimation to classification is that we reformulate the problem of density estimation as an M-estimation task. This reformulation enables the use of a broad class of function spaces (including deep neural networks) to estimate the conditional density function. 
Moreover, since our estimation procedure essentially solves a classification task with a smooth cross-entropy loss function (likelihood function for logistic regression), it lends itself naturally to gradient descent-based optimization techniques. Consequently, the proposed method is computationally efficient and well-suited for practical implementation. 

Under many scenarios such as image generations in generative AI, it is required to generate the data from the estimated density, rather than estimating the density itself.  In this case, the density is implicitly estimated, that is, given the covariate $X=x$, generate $\hat{Y}$ such that the conditional distribution of $\hat{Y}$ given $X=x$ is close to $p_0(Y|X=x)$. 
In this paper, we show that one can further build a sample-efficient implicit density estimator on top of the explicit density estimator illustrated above by leveraging a score-based diffusion model \citep{song2019generative, ho2020denoising}. 
In particular, we establish that one can utilize our estimated explicit density estimator $\hat{p}$ and Monte Carlo sampling to obtain an accurate estimate of the diffused score function. Substituting this estimated score function into the backward process can yield the same error rate in implicit density estimation as the explicit counterpart. 
Furthermore, we rigorously prove that as long as the ground-truth density function can be estimated well 
our proposed reduction from classification (resp. sampling via discretized backward process) can yield explicit (resp. implicit) density estimates at the same error rate. 

Another key advantage of CINDES lies in its ability to use the representational power of deep neural networks, enabling it to effectively adapt to the unknown low-dimensional structure of the density function automatically. We discuss this by examples of the Markov random field and the hierarchical composition model in Section \ref{sec:MRF_HCM}. Briefly speaking, CINDES achieves accelerated convergence rates when the log-density function possesses certain structural properties, such as each variable depends on only a few other coordinates.
As an example, suppose that we observe $Y_1, \dots, Y_n \sim p_0$ and aim to estimate $p_0$, $Y_i \in \reals^{d_y}$. 
If the coordinates of $Y$ are independent and the marginal densities are $\beta$-H\"{o}lder smooth, CINDES can circumvent the curse of dimensionality and estimate $p_0$ at the rate $n^{-\beta/(2\beta + 1)}$.
We summarize our contributions below. 

\begin{enumerate}
    \item We propose a framework for estimating both conditional and unconditional density functions. For the explicit density estimation part, the key idea is to reformulate the density estimation task as a domain classification problem, where we estimate the Bayes classifier distinguishing between real and synthetically generated samples. For the implicit estimation part, the key idea is to show that the explicit estimate of density can further yield an accurate score function estimate. Methodologically, our method is structure-agnostic and computationally efficient, leveraging the efficient implementations of neural network classifications.
    
    \item Theoretically, we show that our proposed procedure can attain the same statistical rate of convergence in explicit and implicit density estimation as if running nonparametric regression when the regression function coincides with the ground-truth density function.
    
    \item As evidence supporting the above claims, we demonstrate that our CINDES estimator algorithmically learns and effectively adapts to low-dimensional structures that neural networks excel at in the (log-)density function, leading to faster convergence rates when such a structure is present, yet is blind to our method.
    
    \item Numerically, we demonstrate the efficacy of our method through extensive simulations and real data analysis.
\end{enumerate}

\noindent \textbf{Organization.} The rest of the paper is organized as follows. In Section \ref{sec:main_method}, we introduce the problem setup, provide a relevant background, and describe our proposed methodology. 
Theoretical properties of the estimator are established in Section \ref{sec:theory}. 
Section \ref{sec:MRF_HCM} presents several examples that demonstrate the effectiveness of our method in estimating structured (un)conditional density functions. 
We conduct extensive simulation studies in Section \ref{sec:simulation} to compare the performance of our approach with other state-of-the-art density estimation methods. In Section \ref{sec:real_data}, we illustrate the practical utility of our method through a real data application. All proofs and additional technical details are provided in the Appendix.

\noindent \textbf{Notation.} We use the upper case $(X, Y)$ to represent random variables/vectors and denote their instances as $(x,y)$. Define $[n]=\{1,\ldots, n\}$. For a vector $x = (x_1,\ldots, x_d)^\top\in \mathbb{R}^d$, we let $\|x\|_2=(\sum_{j=1}^d x_j^2)^{1/2}$. 
We let $a\lor b = \max \{a, b\}$ and $a\land b = \min\{a, b\}$. We use $a(n) \lesssim b(n)$, $b(n) \gtrsim a(n)$, or $a(n) = O(b(n))$ if there exists some constant $C>0$ such that $a(n) \le Cb(n)$ for any $n \ge 3$. Denote $a(n) \asymp b(n)$ if $a(n)\lesssim b(n)$ and $a(n) \gtrsim b(n)$.

\section{Explicit and Implicit Density Estimation}
\label{sec:main_method}
In this section, we present our methodology for both implicit and explicit density estimation using deep neural networks. For the reader’s convenience, the section is organized into three parts: the problem setup is described in Section~\ref{sec:setup}, relevant background is provided in Section~\ref{sec:background}, and our proposed methodology is detailed in Section~\ref{sec:methods}.

\subsection{Setup}
\label{sec:setup}
We consider a supervised learning framework where we observe $n$ i.i.d. pairs \\ $\cD_n = \{(X_1, Y_1), \ldots, (X_n, Y_n)\}$ sampled from the joint distribution of $(X, Y) \in \mathcal{X} \times \mathcal{Y} \subseteq \mathbb{R}^{d_x} \times \mathbb{R}^{d_y}$, where $X \sim \mu_{0,x}$ denotes the marginal distribution of the covariate, and $Y | X = x \sim p_0(y|x)$ denotes the conditional distribution of the response given the covariate.
Let $\mathsf{d}(p, q)$ denote a pre-specified divergence or distance measure between two probability distributions $p(\cdot)$ and $q(\cdot)$ on $\mathcal{Y}$ --- for instance, total variation distance, Hellinger distance, Kullback-Leibler divergence, or a more general $f$-divergence.
Given an estimator $\hat{p}$ of the conditional density $p_0$ and a divergence measure $\mathsf{d}$, we define the \emph{average risk} of the estimator by 
\begin{align}
\label{eq:avg_risk}
    \mathsf{R}_{\mathsf{d}}(p_0, \hat{p}) = \mathbb{E}_{X\sim \mu_{0, x}}\left[\mathsf{d}\big( p_0(\cdot|X), \hat{p}(\cdot|X)\big)\right] = \int \mathsf{d}\big( p_0(\cdot|x), \hat{p}(\cdot|x)\big) \mu_{0, x}(dx),
\end{align}
A standard divergence measure that we consider in this paper is the \emph{total variation} (TV) distance, $\mathsf{TV}(p, q) = \int |p(y) - q(y)| dy$. 
Other examples include several widely used divergence-based measures—such as the $\chi^2$-divergence and Kullback-Leibler divergence --- all of which are special cases of the broader class of $f$-divergences. For the convenience of the readers, we provide the definition of $f$-divergence below: 
\begin{definition}
\label{def:f-divergence}
Given an univariate convex function $f \in [0,\infty)\to (-\infty, \infty]$ satisfying 
\begin{enumerate}
    \item $|f(x)| < \infty$ for any $x>0$ and $f(1) = 0$,
    \item $f(x)$ has uniformly bounded second derivative in $(\epsilon, 1/\epsilon)$ for any $\epsilon>0$, 
\end{enumerate}
we define a $f$-divergence on the space of probability measures as: $\mathsf{D}_{f}(p,q) = \int_{\mathcal{Y}} f(p(y)/q(y)) q(y) dy$. The associated risk of $\hat p$ is defined as $\mathsf{R}_{\mathsf{D}_f}(p_0, \hat p) = \int \mathsf{D}_f(\hat p(\cdot \mid x), p_0(\cdot \mid x)) \ \mu_{0, x}(dx)$.  
\end{definition}
The choices of $f(t) = t \log t$, $f(t) = (t-1)^2$, $f(t) = |t-1|$, and $f(t) = (\sqrt{t} - 1)^2$ correspond, respectively, to Kullback–Leibler (KL) divergence, Pearson $\chi^2$-divergence, total variation distance, Hellinger distance.
In the explicit density estimation setting, our objective is to find $\hat{p}_0(y|x)$, an estimator of $p_0(y|x)$, based on the data set $\cD_n$ such that the averaged risk is small. 
For the implicit density estimation problem, our goal is to simulate the data such that its conditional distribution given $X$ is as close to that of $\hat Y$ given $X$ as in the applications of generative AI with guidance.  More precisely, we want to learn a transformation $\hat{h}$ that maps a known noise distribution $U \sim \mu_U$ (typically Gaussian or uniform) and the covariate $X$ to a synthetic output $\hat Y = \hat h(X, U)$, such that the conditional distribution of $\hat Y$ given $X$ closely approximates the true conditional distribution of $Y$ given $X$, based on the observed data $\cD_n$. We use the notation $p_{\hat{Y}(U)|X}(y|x)$ to denote this estimated conditional density of $Y$ given $X$ induced by $\hat h$. 
We use the subscript $\hat{Y}(U)|X$ to emphasize that the implicit density estimator does not yield an explicit form of the conditional density $p_0$, but instead approximates it through generated samples. The resulting distribution $p_{\hat{Y}(U)|X}(y|x)$ is determined by both the transformation $\hat{h}$ and the noise distribution $\mu_U$. Here we also use $U$ to emphasize that the randomness of $\hat{Y}$ given fixed $X$ comes from the noise $U$ it injects, and will abbreviate it when the defined noise $U$ is clear from context. 
Similar to the explicit density estimation setup, here also we evaluate the performance of the implicit conditional density estimator through its average risk, albeit the average is now taken over the distribution of $X \sim \mu_{0, x}$: 
\begin{align}
\mathsf{R}_{\mathsf{d}}(p_0, p_{\hat{Y}|X}) = \mathbb{E}_{X\sim \mu_{0,x}}\left[\mathsf{d}\big(p_0(\cdot|X), {p}_{\hat{Y}(U)|X}(\cdot|X)\big)\right]. \label{fan1}
\end{align}
\begin{remark}
We also note that the above setup naturally encompasses unconditional density estimation as a special case. In the unconditional setting, we observe i.i.d. samples $Y_1, \ldots, Y_n \sim p_0(y)$, where $p_0 : \mathbb{R}^{d_y} \to \mathbb{R}^+ \cup {0}$ is an unknown density function. The goal in this case is to estimate $p_0(y)$ based solely on these observations. As before, the performance of unconditional explicit and implicit density estimators is evaluated using a discrepancy measure, which simplifies to $\mathsf{d}(p_0, \hat{p})$ and $\mathsf{d}(p_0, p_{\hat{Y}})$, respectively.
The unconditional density estimation can be written as a special case of the conditional density estimation via setting $X = \phi$ (the null set) and thus $\mu_0(dx, dy) = \mu_0(dy) = p_0(y) dy$. 
\end{remark}

\subsection{Background}
\label{sec:background}
As outlined in the introduction, this paper employs deep neural networks for estimating conditional density functions using both implicit and explicit approaches. In particular, for implicit estimation, we obtain the transformation $\hat{h}$ using a score-based diffusion process.
Before presenting the details of our proposed method, we provide a brief overview of deep neural networks and diffusion models for the ease of the readers. 
In brief, we will utilize deep neural networks as scalable, non-parametric techniques for explicit density estimation, and leverage the concept of time-reversal stochastic differential equations in diffusion models to construct an implicit density estimator based on the explicit neural density estimator. 

\subsubsection{Deep neural networks}
In this paper, we adopt the fully connected deep neural network with ReLU activation function $\sigma(x) = \max\{0, x\}$, and we denote it as \textit{deep ReLU network}. Let $L$ be any positive integer and $ (d_0, \ldots , d_{L+1}) = (d, N, \ldots, N, 1) $ with any positive integer $N$. A \textit{deep ReLU network with width $N$ and depth $L$} is a function mapping from $\mathbb{R}^{d_0}$ to $\mathbb{R}^{d_{L+1}}$ with the form 
\begin{equation}
\label{eq: g composition}
g(x) = \mathcal{L}_{L+1} \circ \bar{\sigma} \circ  \mathcal{L}_{L+1} \circ \bar{\sigma} \circ \cdots \circ \mathcal{L}_{2} \circ \bar{\sigma} \circ \mathcal{L}_1 (x),
\end{equation}
where $\mathcal{L}_i(x) = W_i x + b_i$ is an affine transformation with the weight matrix $W_i \in \mathbb{R}^{d_i \times d_{i-1}}$ and bias vector $b_i \in \mathbb{R}^{d_i}$, and $\bar{\sigma}: \mathbb{R}^{d_i} \rightarrow \mathbb{R}^{d_i}$ applies the ReLU activation to each entry of a $\mathbb{R}^{d_i}$-valued vector. Here, the equal width is for presentation simplicity. 

\begin{definition}[Deep ReLU network class]
\label{def:nn}
    Define the family of deep ReLU network truncated by $R$ with depth $L$, width $N$ as $\mathcal{H}_{\mathtt{nn}}(d, L, N, R)=\left\{\widetilde{g}= T_R g: g \text{ of form } \eqref{eq: g composition}\right\} $
where $T_R$ is the truncation operator at level $R>0$ to each entry of a vector, defined as $T_R u = \mathrm{sgn}(u)(|u| \wedge R)$. 
\end{definition}

\subsubsection{Diffusion model}
Let $\nu$ be the target distribution on $\mathbb{R}^d$ from which we aim to generate samples. The core idea behind diffusion models is to define a forward process, typically governed by a stochastic differential equation (SDE), that gradually adds noise to samples from $\nu$, transforming them into a simple reference distribution (e.g., Uniform or Gaussian). A corresponding backward process, given by the time-reversed SDE, is then used to transform noise back into samples from $\nu$ (e.g., see \cite{bakry2013analysis} for details). To be specific, for the forward process, we consider the following Ornstein-Uhlenbeck (OU) process
\begin{align}
\label{eq:OU_def}
    dY_t = -\beta_t Y_t dt + \sqrt{2\beta_t} dB_t \qquad X_0 \sim \nu.
\end{align} 
where $\beta_t \in \mathbb{R}^+$ is a time-dependent weighting function to be specified later, and $(B_t)_{t \ge 0}$ denotes a standard Brownian motion in $\mathbb{R}^d$. We use $p_t$ to denote the marginal distribution of $Y_t$, and define the score function of $p_t$ as $\nabla_z \log p_t(z) \in \mathbb{R}^d$, where the $j$-th coordinate is given by $[\nabla_z \log p_t(z)]_j = \partial{z_j} \log p_t(z)$.
For a fixed timestep $T$, it is known from \cite{anderson1982reverse,haussmann1986time} that the backward process of Equation \eqref{eq:OU_def}, $(\breve{Y}_t)_{t\in [0, T]} = (Y_{T-t})_{t\in [0,T]}$, satisfies the following SDE
\begin{align}
\label{eq:backward}
    d\breve{Y}_t = \beta_{T-t} \left(\breve{Y}_t + 2\nabla_z \log p_{T-t}(\breve{Y}_t)\right)dt + \sqrt{2\beta_{T-t}} dW_t \qquad \breve{Y}_0 \sim p_T \,,
\end{align} 
where $(W_t)_{t\in [0,T]}$ is another Brownian motion. 
Given $p_T$ converges to a standard normal distribution exponentially fast when $T\to \infty$, we can generate samples through the SDE in \eqref{eq:backward} and initialization $\breve{Y}_0 \sim \mathcal{N}(0, I_d)$ if the ground truth diffused score function $\nabla_z \log p_{T-t}(\cdot)$ is known. 
In practice, one typically estimates the score function from the observed data and then discretizes the backward process \eqref{eq:backward} to transform a noise sample into a draw from the target distribution.  For an overview, see \cite{tang2025score}.
 

\subsection{Our Method}
\label{sec:methods}
Given the observations $\cD_n = {(X_i, Y_i)}_{i=1}^n$, the first step of our method involves generating a synthetic sample of ``fake'' responses $\tilde{Y}_1, \ldots, \tilde{Y}_n$ independently drawn from the uniform distribution over $\mathcal{Y}$. 
In the second step, we perform logistic regression on the combined dataset ${(Z_i, 1)}_{i=1}^n \cup {(Z_{i}, 0)}_{i=n+1}^{2n}$, where each $Z_i \in \mathbb{R}^{d_x + d_y}$ is defined as, $Z_i =  [Y_i, X_i]$ for $1 \le i \le n$ and $Z_i = [\tilde{Y}_i, X_i]$ for $i > n$.
Here, by $[y, x]$, we mean the concatenation of two vectors $y\in \mathbb{R}^{d_y}$ and $x\in \mathbb{R}^{d_x}$. 
Let $\sigma(t) = \frac{1}{1 + e^{-t}}$ denote the sigmoid function. The explicit conditional density estimator of $p_0(y\mid x)$ is defined as the exponential of the following empirical risk minimizer:
\begin{align} \label{fan2}
    \hat{p}(y|x) = \exp(\hat{f}(y,x)) \cdot \int_{\mathcal{Y}} dy ~~~ &\text{where} ~~~ \hat{f} \in \argmin_{f\in \mathcal{H}_{\mathtt{nn}}(d_y+d_x, L, N, R) } \hat{\mathsf{L}}(f),
\end{align} 
where the collection of neural networks $\mathcal{H}_{\mathtt{nn}}(d_y+d_x, L, N, R)$ is defined in Definition \ref{def:nn}, and 
the loss function $\hat{\mathsf{L}}(f)$ is defined as: 
\begin{align}
\label{eq:loss}
    \hat{\mathsf{L}}(f) = \frac{1}{n} \sum_{i=1}^n \left[-\log(\sigmoid(f(Y_i, X_i))) -\log(1-\sigmoid(f(\tilde{Y}_{i}, X_{i})))\right].
\end{align} 
To understand the intuition behind the loss function in  Equation \eqref{eq:loss}, consider the corresponding limiting population loss as $n \uparrow \infty$; it is immediate from the law of large numbers that: 
\begin{equation*}
\textstyle
    \hat{\mathsf{L}}(f) \overset{P}{\longrightarrow} \mathsf{L}(f) \triangleq \mathbb{E}_{X, Y, \tilde Y}\left[-\log \sigma(f(X,Y))  - \log (1- \sigma(f(X, \tilde{Y})))\right].
\end{equation*} 
Furthermore, one can also show that the logarithm of $p_0(y\mid x)$ minimizes the population loss over the space of all measurable functions: 
\begin{equation*}
\textstyle
f^\star = \log \left[p_0(y|x) / \int_{\mathcal{Y}} dy\right]  = \argmin_{f} \ \mathsf{L}(f) \,.
\end{equation*}
Therefore, under a suitable choice of model complexity hyperparameters for the function class $\mathcal{H}_{\mathtt{nn}}(d_y + d_x, L, N, M)$, standard arguments for the consistency of $M$-estimators suggest that $\hat{p}(y \mid x)$ consistently estimates the true conditional density $p_0(y \mid x)$ as $n \to \infty$. We summarize our procedure for constructing this explicit density estimator in \cref{algo:ede}.

We note that the use of the uniform distribution for generating $\tilde Y$ is not essential. In general, any reference distribution whose support contains $\mathcal{Y}$ and whose density on $\mathcal{Y}$ is bounded away from both $0$ and $\infty$ can be employed, as this condition ensures the stability of the density ratio. Thus, the normal or $t$-distributions are equally valid. Usually, we would like the covariance, denoted by $\Sigma_0$, of the reference distribution to be similar to that of the data distribution.  For the low-dimensional case, we can take $\Sigma_0$ as the sample covariance, and for the high-dimensional case, we can take a regularized covariance matrix, such as POET \citep[low-rank plus diagonal version to ensure positive definiteness]{fan2013large}, as $\Sigma_0$.  The resulting estimator \eqref{fan2} needs to be multiplied by the reference density.
In our paper, we mainly adopt the uniform distribution for simplicity in presenting the technical results.

\begin{algorithm}[!t]
\caption{Neural Explicit Density Estimator}
\begin{algorithmic}[1]
\State \textbf{Input:} Data $\mathcal{D} = \{(X,Y)\}_{i=1}^n$.
\State \textbf{Input:} Neural network hyper-parameters $L, N, R$ in \cref{def:nn}.
\State Draw $n$ i.i.d. fake responses $\tilde{Y}_1, \ldots, \tilde{Y}_n$ from $\mathrm{Uniform}(\mathcal{Y})$.
\State Run empirical risk minimization $\hat{f} \in \argmin_{f\in \mathcal{G}(d_y+d_x, L, N, R) } \hat{\mathsf{L}}(f)$ with loss $\hat{\mathsf{L}}$ defined in \eqref{eq:loss}.
\State \textbf{Output:} $\hat{p}(y|x) = \exp(\hat{f}(y, x))$ (with normalization (optional), see Remark~\ref{remark4}). 
\end{algorithmic}
\label{algo:ede}
\end{algorithm}
\vspace{1em}

\noindent \textbf{Implicit density estimation and sample generation. } Having outlined our approach to explicit conditional density estimation using neural networks, we now turn to the task of implicit conditional density estimation. This procedure begins with an estimate of the true conditional density $p_0(y \mid x)$—for example, the explicit estimator $\hat{p}(y \mid x)$ obtained via Algorithm~\ref{algo:ede}. Using this estimate, we approximate the score function of the diffused distribution, which is then plugged into the backward process described in Equation~\eqref{eq:backward} to generate new samples from the target distribution. 
Let us now elaborate on this step. 
Throughout this paper, we adopt a constant weighting function $\beta_t = 1$. 
We fix $x \in \cX$, and then start with the forward diffusion process (Equation \eqref{eq:OU_def}) $Y_0 \sim p_0(y \mid X = x)$: 
\begin{align*}
    dY_t = -Y_t dt + \sqrt{2} dB_t \qquad Y_0 \sim p_0(\cdot|X=x),
\end{align*} 
The conditional distribution of $Y_t$ given $Y_0$ can be written as: 
\begin{align*}
    Y_t|Y_0 \sim \mathcal{N}\left(m_t Y_0, \sigma_t^2 I_{d_y}\right) \qquad \text{with} \qquad m_t = e^{-t}, \ \sigma_t = \sqrt{1-e^{-2t}}.
\end{align*} 
Denote $p_t(\cdot|x)$ as the induced density of $Y_t$ given $X=x$. The change-of-variable formula yields: 
\begin{equation}
\textstyle
\label{eq:score_func_def}
    s^\star(y, t|x) := \nabla_y \log p_t(y|x) = \frac{1}{\sigma_t^2} \frac{\mathbb{E}_{U\sim N(0,I_{d_y})} \left[U \cdot  p_0\left(\frac{y-\sigma_t \cdot U}{m_t} \big| x\right)\right]}{\mathbb{E}_{U\sim N(0,I_{d_y})} \left[  p_0\left(\frac{y-\sigma_t \cdot U}{m_t} \big| x\right)\right]}.
\end{equation} 
Given $\hat{p}$, an estimator of $p_0$, we can adopt the following plug-in-based estimation
\begin{equation}
\textstyle
\label{eq:score-est}
    \hat{s}_K(y, t|x) = \frac{1}{\sigma_t^2} \frac{\frac{1}{K} \sum_{k=1}^K \left[U_{t,k} \cdot  \hat{p}\left(\frac{y-\sigma_t \cdot U_{t,k}}{m_t} \big| x\right)\right]}{\frac{1}{K} \sum_{k=1}^K \left[  \hat{p}\left(\frac{y-\sigma_t \cdot U_{t,k}}{m_t} \big| x\right)\right]} \qquad \text{with} \qquad U_{t,1}\ldots, U_{t,K} \overset{\mathrm{i.i.d.}}{\sim} \mathcal{N}(0, I_{d_y})
\end{equation} to estimate the score function.
Observe that here is an additional layer of randomness in $\hat{s}(\cdot,t|x)$ given observed data $\cD_n$, which comes from the simulated $U_{t,1},\ldots, U_{t,K}$. However, these random variables are independent of both $\cD_n$ and the choice of $(x, y)$. 

We are now ready to present our implicit density estimator. 
As mentioned earlier, this estimator is constructed by discretizing the backward diffusion process, and towards that, we follow the scheme presented in \cite{huang2024denoising}.  
Let $T > 0$ denote the terminal time of the forward diffusion process, and let $\delta \in (0,1)$ represent the early-stopping threshold for the backward process. Specifically, we run the backward process from time $T$ down to time $\delta$, where $\delta$ is chosen to be close to zero. 
Furthermore, let $M$ be the number of discretization steps for the backward process. We pick the discretization timesteps $t_0, t_1,\ldots, t_{M}$ as: 
\begin{align}
\label{eq:discrete-timestep}
    t_m = \begin{cases}
        \frac{(T-1)m}{M/2} & \qquad m \le M/2\\
        T-\delta^{(2m-M)/M} & \qquad m > M/2
    \end{cases}
\end{align} 
The first half of the timesteps (i.e., $1 \le m \le M/2$) is picked uniformly from $[0, T-1]$, and the second half of the timesteps (i.e., $M/2 < m \le M$) grows exponentially in $[T-1, T-\delta]$. 
The two time intervals correspond to $[1,T]$ and $[\delta, 1]$ respectively in the forward process $\{Y_t\}$. 
Given a fixed $x$ (recall that we aim to generate sample from $p_0(y \mid x)$), we use the following $M$-step discretized SDE to generate $\hat{Y} := \breve{Y}_{t_M}$:  
\begin{align}
\begin{split}
    \breve{Y}_0 &\gets W_0 \\
    \breve{Y}_{t_{m+1}} &\gets \frac{1}{\alpha_m} \left[\breve{Y}_{t_{m}} + (1-\alpha_m) \hat{s}(\breve{Y}_{t_m}, T-t_m|x) \right] + \sqrt{\frac{(1-\alpha_m)(1-\bar{\alpha}_m)}{1-\bar{\alpha}_{m+1}}} W_m
\end{split}
\end{align} where $W_0,\ldots, W_M$ are i.i.d. $\mathcal{N}(0, I_{d_y})$ distributed variables that are independent of both $\cD_n$ and the random variables $\{U_{t_m, k}\}_{m\in [M] \cup \{0\}}$, independent of $\cD_n$, used to construct $\hat s_k$ for $1 \le k \le K$, as proposed in Equation \eqref{eq:score-est}. 
The coefficients $(\alpha_m, \bar{\alpha}_{m})$ are defined as: 
\begin{align}
\label{eq:alpha-def}
\alpha_m = e^{-2(t_{m+1} - t_m)} \qquad \text{and} \qquad \bar{\alpha}_m = e^{-2(T - t_m)}.
\end{align} 
We expect that distribution of $\hat{Y} := \breve{Y}_{t_M}(\{W_m, U_{t_m, 1},\ldots, U_{t_m, K}\}_{m=1}^M)$ to be close to $p_\delta(\cdot|x)$, whose distribution is also close to $p_0(\cdot|x)$ when $\delta$ is small. 
See the entire procedure in \cref{algo:ide}.

\begin{algorithm}[!t]
\caption{Neural Implicit Density Estimator}
\begin{algorithmic}[1]
\State \textbf{Input:} Explicit density estimator $\hat{p}(\cdot|x)$ and fixed $X=x$. 
\State \textbf{Input:} Diffusion Hyper-parameters $T$, $\delta$, discretization hyper-parameter $M$, $K$.
\State Sample $W_0,\ldots, W_M$ from $\mathcal{N}(0, I_{d_y})$.
\State Initialize $\breve{Y}_0 \gets W_0$.
\For{$m \in \{0,\ldots, M-1\}$}
    \State Sample $\mathcal{U}_{t_m} = \{{U}_{t_m, 1},\ldots, {U}_{t_m, K}\} \sim \mathcal{N}(0, I_{d_y})$.
    \State Calculate $\hat{s}_{K,m} \gets \hat{s}_K(\breve{Y}_{t_m}, T-t_m|x)$ by \eqref{eq:score-est} using $\hat{p}$ and $\mathcal{U}_{t_m}$.
    \State $\breve{Y}_{t_{m+1}} \gets \frac{1}{\alpha_m} \left[\breve{Y}_{t_{m}} + (1-\alpha_m) \hat{s}_{K,m} \right] + \sqrt{\frac{(1-\alpha_m)(1-\bar{\alpha}_m)}{1-\bar{\alpha}_{m+1}}} W_m$ with $\alpha_m$ and $\bar{\alpha}_m$ in \eqref{eq:alpha-def}.
\EndFor
\State \textbf{Output:} $\hat{Y} = \breve{Y}_{t_M}$.
\end{algorithmic}
\label{algo:ide}
\end{algorithm}

\section{Theory}
\label{sec:theory}
In this section, we present our main theoretical results, which characterize the estimation errors of both the explicit and implicit conditional density estimators for $p_0(y \mid x)$.
We assume to have access to data $\mathcal{D}_n = \{(X_i, Y_i)\}_{i=1}^n \stackrel{i.i.d}{\sim} (X,Y)$, where $X \sim \mu_{0, x}$ and $Y|X=x\sim p_0(y|x=x)$. We further collect samples $\tilde{Y}_1, \dots, \tilde Y_n \stackrel{i.i.d}{\sim}$ Uniform$(\mathcal{Y})$ independent of $\mathcal{D}_n$. For simplicity, we assume $\mathcal{Y} = [0,1]^{d_y}$. 
We begin by stating two conditions required to establish our theoretical results. 

\begin{condition}[Distributions of the observations]
\label{cond:regularity}
We assume the covariates to have bounded support; in particular, we assume $\|X\|_\infty \le 1$. Furthermore, given any $X = x$, we assume that the conditional density $p_0(y|X = x)$ is supported on $[0,1]^{d_y}$ and bounded from away from zero and infinity, namely satisfing  $\sup_{y\in \mathcal{Y}} \max\{p_0(y|X), 1/p_0(y|X)\} \le c_1$  $X$-a.s., where $c_1$ is some universal constant.
\end{condition}
\begin{condition}[Neural network truncation hyper-parameter]
\label{cond:trunc}
We use $\mathcal{G} = \mathcal{H}_{\mathtt{nn}}(d_x+d_y, L, N, R)$ with $\log(c_1) \le R \le \log(c_2)$ for some constant $c_2$. This lower bound ensures that the collection of neural networks is large enough to learn $\log p_0$, while remaining bounded. 
\end{condition}

\begin{remark}
    Condition~\ref{cond:regularity} is a standard assumption in the nonparametric estimation literature. The boundedness of the domain of $(X, Y)$ is mainly adopted for notational and conceptual simplicity. Although we assume $\cX = [-1, 1]^{d_x}$ and $\cY = [0, 1]^{d_y}$ throughout the paper, our results and analysis extend naturally to the case where $\cX$ and $\cY$ are compact subsets of $\mathbb{R}^{d_x}$ and $\mathbb{R}^{d_y}$, respectively. Moreover, the analysis can be further generalized to unbounded domains using standard truncation arguments. Since such an extension does not introduce any new conceptual insights, we omit it for the sake of clarity of exposition. 
   The assumption that the conditional density $p_0(y \mid x)$ is bounded above and below is also standard in the literature, as it ensures that the log-density remains bounded. Nevertheless, this assumption can be relaxed using a similar truncation-based argument applied to the log density, without altering the key ideas.  
\end{remark}

Having stated the necessary conditions, we are now ready to present our main theorems, which provide non-asymptotic error bounds for both the implicit and explicit conditional density estimators. 
For any function $f(y, x)$, we define the $L_2$ norm as
\begin{align}
\label{eq:l2}
    \|f\|_2 = \sqrt{\int \left(\int \left|f(y, x)\right|^2 dy\right) \mu_{0,x}(dx)}
\end{align}
Note this is the same as the standard $L_2$ norm with respect to the product of Lebesgue measure on $\cY$ and $\mu_{0, x}$ on $\cX$.  
The following theorem provides an oracle-type inequality for our explicit density estimator in \cref{algo:ede} in a structure-agnostic manner: 
\begin{theorem}[Explicit density estimator]
\label{thm:main}
    Assume Conditions \ref{cond:regularity} and \ref{cond:trunc} hold. Then for any $n\ge 3$ and $t>0$, the following event
    \begin{align}
        \|\hat{p} - p_0\|_2^2 \le C\left\{\inf_{g\in \mathcal{G}} \|g - \log p_0\|_2^2 + \frac{(NL)^2 \log (n) + t}{n}\right\} =:{\delta_{{stat}}},
    \end{align} occurs with probability at least $1-2e^{-t}$, where $C$ is a constant depending polynomially on $c_2$.
\end{theorem}

Theorem~\ref{thm:main} establishes a high-probability bound on the deviation between the explicit conditional density estimator $\widehat{p}$ and the true conditional density $p_0$. As is evident from the result, the error bound $\delta_{\text{stat}}$ consists of two terms, mirroring the typical structure found in standard nonparametric regression: i) the neural network approximation error $\inf_{g \in \mathcal{G}}\|g -\log p_0\|_2^2$ to the underlying density $p_0$, and ii) the stochastic error with $n^{-1}\log n$ and $(NL)^2$ that relies on the Pseudo-dimension of the neural network class we used. 
Both of these components depend on the hyperparameters $(N, L)$ of the underlying neural network class. Increasing these parameters reduces the approximation error but simultaneously increases the stochastic error, reflecting the classic bias–variance trade-off. 
As a consequence, when the ground truth $p_0$ lies within some smooth function class (e.g., H\"{o}lder, Sobolev, etc.), an optimal rate can be achieved by choosing appropriate $N$ and $L$ to trade off both the approximation error and the stochastic error. 

\begin{remark}
    While our result is presented for the deep ReLU network class $\mathcal{G}$, the same result applies to other function classes, or generic machine learning models. To be specific, let $\mathcal{F}$ be any bounded function class whose critical radius of local Rademacher complexity is $\delta_{\mathtt{s}}$, that is, $\mathsf{Rade}(\delta; \partial \mathcal{F}) \le \delta \delta_{\mathtt{s}}$ for any $\delta \ge \delta_{\mathtt{s}}$, where
    \begin{align*}
        \mathsf{Rade}(\delta; \partial \mathcal{\mathcal{F}}) := \mathbb{E}_{\{(X_i,Y_i)\}_{i=1}^n, \{\varepsilon_i\}_{i=1}^n}\left[\sup_{\substack{f,\tilde{f}\in \mathcal{F} \\ \|f - \tilde{f}\|_2\le \delta}}\frac{1}{n} \sum_{i=1}^n \varepsilon_i (f-\tilde{f})([X_i,Y_i]) \right]
    \end{align*} and where $\{(X_i,Y_i)\}_{i=1}^n$ are i.i.d. drawn from $(X,Y)$ with $X\sim \mu_x$ and $Y|X=x\sim p_0(\cdot|x)$, and $\{\varepsilon_i\}_{i=1}^n$ are independent Rademacher random variables that is also independent of $\{(X_i,Y_i)\}_{i=1}^n$. Then we have $\hat{p}$ minimizing \eqref{eq:loss} within the class $\mathcal{F}$ satisfies
    \begin{align*}
        \|\hat{p} - p_0\|_2^2 \lesssim \inf_{f\in \mathcal{F}} \|f - \log p_0\|_2^2 + \delta_{\mathtt{s}}^2 + \frac{t + \log n}{n},
    \end{align*} with probability at least $1-2e^{-t}$; see a full statement in \cref{sec:proof-explicit-main}. This includes 
    a wide array of nonparametric function classes of interest, including but not limited to, spline methods, RKHS with bounded norm \citep{friedman1991multivariate, wahba1990spline}. 
\end{remark}

\begin{remark}\label{remark4}
    In general, the estimator $\hat p$ produced by Algorithm \ref{algo:ede} may not be perfectly normalized. To address this, one can normalize it (up to an arbitrarily small error) using a simple Riemann integration approach: sample $Y_1, \dots, Y_k \sim \mathrm{Unif}(\mathcal{Y})$ and then update
    $$
    \hat p_{\rm norm}(y \mid x) \leftarrow \frac{\hat p(y \mid x)}{\frac{\mathrm{Vol}(\cY)}{k} \sum_{j = 1}^k \hat p(Y_j \mid x)} \,.
    $$
    By choosing $k$ sufficiently large, the approximation error in the normalizing constant can be made arbitrarily small (it scales as $k^{-1/2}$).  
\end{remark}
We now present a direct corollary of Theorem \ref{thm:main}, which provides the upper bound on the estimation error of $\hat p$, but with respect to $\mathrm{TV}$ distance, and for general $f$-divergence: 

\begin{corollary}
\label{cor:main}
Recall the definition of $\mathsf{R}_{\mathsf{TV}}$ and $\mathsf{R}_{\mathsf{D}_f}$ in Definition \ref{def:f-divergence}.
Under the setting of \cref{thm:main}, we have
\begin{align*}
    \mathsf{R}_{\mathsf{TV}}(p_0, \hat{p}) \lesssim \sqrt{\delta_{\text{stat}}}, \quad \mathsf{R}_{\mathsf{D}_f}(p_0, \hat{p}) \lesssim \sqrt{\delta_{\text{stat}}}, \quad \text{and}  \quad  \mathsf{R}_{\mathsf{D}_f}(p_0, \hat{p}_{\rm norm}) \lesssim \delta_{\text{stat}} \,,
\end{align*}
under the same high probability event as in Theorem \ref{thm:main}. 
\end{corollary}

We now present our theoretical results for implicit density estimation. Recall that in our implicit density estimation procedure (Algorithm \ref{algo:ide}), we rely on an explicit density estimator $\hat p_0$, assumed to satisfy the error bound in Theorem \ref{thm:main}, to estimate the score function. Consequently, this procedure involves three primary sources of error:
\begin{enumerate}
    \item {\bf Score estimation: }The first source of error arises from the estimation of the score function. This, in turn, has two contributing factors: i) the estimation error of $\hat p_0$, and ii) the finite-sample Monte Carlo approximation error—namely, the discrepancy introduced by replacing the Gaussian expectation in Equation \eqref{eq:score-est} with its Monte Carlo average. 

    \item {\bf Discretization error: }The second source of error stems from the discretization of the continuous stochastic differential equations. This error is unavoidable, as continuous SDEs cannot be simulated exactly on a machine with finite precision.
    
    \item {\bf Time truncation error: }The third source of error arises from early stopping and truncating the time horizon. Ideally, running the forward (resp. backward) process until time $T = \infty$ would yield the standard Gaussian distribution (resp. the true data-generating distribution). In practice, however, the process is terminated at a finite (albeit large) time $T$, introducing a truncation error.

\end{enumerate}
As evident from the above discussion, the second and third sources of error are inherent to the diffusion-based generative process and are not specific to our methodology. 
These types of errors have been studied in prior works (e.g., \cite{benton2023nearly,huang2013oracle}), where various error bounds have been established; see also \cite{tang2025score}. 
However, the error arising from the estimation of the score function is specific to our methodology and thus requires a new analysis. 
The following proposition provides a non-asymptotic error bound for the score function estimator in Equation \eqref{eq:score-est}:
\begin{proposition}
\label{prop:score_error_bound}
    Recall the plug-in diffused score estimator in \eqref{eq:score-est} with $\|\hat{p} - p_0\|_2^2 \le \delta_{stat}$, we have for any $t>0$, with probability $\ge 1 - K^{-100}$,  
    \label{thm:score_est}
    \begin{align}
    \label{eq:score_bound}
      & \mathbb{E}_X \left[\int (\hat{s}_K(y, t|X) - s^\star(y, t|X) )^2 \ p_t(y|X) \ dy\right] \notag \\
      & \qquad \qquad \le \Upsilon \left\{\frac{d_y \|p_0^{-1}\|_\infty \left(1 + \|\hat{p}\|_\infty\right)}{\sigma_t^2}\delta_{\rm stat} + \frac{d_y^2 (\log{K})^2\|\hat p_0\|_\infty^4\|\hat p^{-1}_0\|_\infty^2}{K\sigma^2_t}\right\}  \triangleq \delta_{\rm score}(t) \,.
    \end{align}
    for some universal constant $\Upsilon > 0$.
\end{proposition}

\begin{remark}
The above proposition provides an upper bound on the estimation error of the score function at any fixed time point $t$. This bound consists of two main components:
(i) the first term in Equation \eqref{eq:score_bound} captures the error arising from the estimation of the explicit density $\hat p_0$, and
(ii) the second term reflects the error due to Monte Carlo approximation of the standard Gaussian expectation using $K$ samples (as shown in Equation \eqref{eq:score-est}).
Importantly, the parameter $K$ is user-controlled, and the second term can be made arbitrarily small by choosing a sufficiently large $K$, at the cost of increased computation for evaluating $\hat p_0$ at $K$ points and averaging the results (which is effectively negligible unless $K$ is extremely large). As a result, the first term typically dominates the overall error, implying that the estimation error of the score function is essentially proportional to that of the explicit density estimator $\hat p_0$. 
\end{remark}

In Proposition \ref{prop:score_error_bound}, we established a non-asymptotic error bound for estimating the score function at any fixed time $t$. However, as outlined in Algorithm \ref{algo:ide}, we need to run the backward process for $M$ discrete time steps. By applying a union bound over the time indices to the bound in Proposition \ref{prop:score_error_bound}, we obtain the following corollary: 
\begin{corollary}
\label{cor:score_est_union}
    Under the setup in Proposition \ref{prop:score_error_bound}, with probability $\ge 1 - MK^{-100}$: 
    $$
    \mathbb{E}_X \left[\int (\hat{s}_K(y, t_j|X) - s^\star(y, t_j|X) )^2 \ p_t(y|X) \ dy\right] \le \delta_{\rm score}(t_j) \,, \qquad \forall \ 1 \le j \le M. 
    $$
\end{corollary}
We now present our main theorem on the estimation of the conditional density $p_{\hat Y(U) \mid X}$, which aggregates the errors from all three sources discussed above to provide a non-asymptotic bound on the estimation error of the implicit density estimator: 
\begin{theorem}[Implicit density estimator]
\label{thm:main2}
    Recall the risk of implicit density estimator defined in \eqref{fan1}.
    Under the event of \cref{thm:main} and Corollary \ref{cor:score_est_union}, the implicit density estimator in \cref{algo:ide} with $\hat{p}$ being that used in \cref{thm:main} satisfies
    $\mathsf{R}_{TV}(p_0(\cdot|X), p_{\hat{Y}|X}) \le \sqrt{C\delta_{\text{all}}}$ and $\mathsf{R}_{KL}(p_0(\cdot|X), p_{\hat{Y}|X}) \le C\delta_{\text{all}}$,
    where
    \begin{align*}
        \delta_{\rm{all}} = \sum_{n=0}^{M}(t_{n+1} - t_n) \delta_{\rm score}(T - t_n) + \delta + d_y e^{-2T} + d_y \frac{[T + \log(1/\delta)]^2}{M}
    \end{align*}
and $C > 0$ is some constant independence of $(c_1, c_2, K, n, d_y, T, M)$. In particular if we take $T \asymp \log n$, $K \asymp \delta_{\text{stat}}^{-1}$, $M \asymp \delta_{\text{stat}}^{-1}$, and $\delta \asymp \delta_{\text{stat}}$, then we have: 
$$
    \mathsf{R}_{TV}(p_0(\cdot|X), p_{\hat{Y}|X}) \le \sqrt{C_1 \delta_{\text{stat}}}\log{n}, \qquad \text{and} \qquad \mathsf{R}_{KL}(p_0(\cdot|X), p_{\hat{Y}|X}) \le C_1 \delta_{\text{stat}} \log^2{(n)} \,.
    $$
for some constant $C_1 > 0$. 
\end{theorem}
It is instructive to examine and interpret the different components of $\delta_{\rm all}$, which, as discussed earlier, reflects the combined effect of three distinct sources of errors. The first term captures the aggregated estimation error of the score function over the time points $\{T - t_m\}_{m\in [M]}$. The second and third terms together account for the error introduced by early stopping and finite-time truncation of the SDE. Finally, the fourth term arises from the discretization of the SDE.

\section{Examples and Convergence Rates}
\label{sec:MRF_HCM}

In real-world scenarios, the target density often exhibits low-dimensional structures, such as factorized or compositional forms. This section demonstrates that our estimators automatically adapt to such different hidden structures and consequently achieve faster convergence rates without the knowledge of these structures. 
Before delving deep into the discussion, we first introduce the notion of H{\"o}lder smooth function: 

\begin{definition}[$(\beta,C)$-smooth Function]
\label{def:holder_func}
    Let $\beta = r+s$ for some nonnegative integer $r\ge 0$ and $0<s\le 1$, and $C>0$. A $d$-variate function $f$ is $(\beta, C)$-smooth if for every non-negative sequence $\alpha \in \mathbb{N}^d$ such that $\sum_{j=1}^d \alpha_j = r$, the partial derivative $\partial^{\alpha} f=(\partial f)/(\partial z_1^{\alpha_1}\cdots z_d^{\alpha_d})$ exists and satisfies $|\partial^{\alpha} f(z) - \partial^{\alpha} f(\tilde{z})| \le C \|z - \tilde{z}\|_2^s$.
    We use $\mathcal{H}_{\mathtt{h}}(d, \beta, C)$ to denote the set of all the $d$-variate $(\beta, C)$-smooth functions.
\end{definition}
In a nutshell, if a function $f \in \mathcal{H}_{\mathtt{h}}(d, \beta, C)$, then the function is differentiable up to order $\lfloor \beta \rfloor$ with uniformly bounded derivatives, and its $\lfloor \beta \rfloor$-th derivative is Lipschitz of order $\beta - \lfloor \beta \rfloor$. In the following sections, we present several examples of structured density functions and demonstrate the fast convergence rates achieved by CINDES.

\subsection{Low-dimensional factorizable structure}

We first consider the following structure, whose density is the product of $d^\star$-variate $(\beta,C)$-smooth functions with $d^\star \le d$. 
\begin{definition}
Let $\beta, C \in \mathbb{R}^+$ and $d, d^\star \in \mathbb{N}^+$ satisfying $d^\star\le d$. We define $\mathcal{H}_{\mathsf{f}}(d, d^\star, \beta, C)$ as
\begin{align*}
    \mathcal{H}_{\mathsf{f}}(d, d^\star, \beta, C) = \Bigg\{h(z) = \prod_{|J|\le d^\star} f_J(z_J): f_J\in \mathcal{H}_{\mathtt{h}}(|J|, \beta, C)\Bigg\}.
\end{align*}
\end{definition}

The above definition implies that $\cH_{\mathsf{f}}(d, d^*, \beta, C)$ consists of all density functions that can be factorized into multiple components, where each component depends on approximately $d^*$ variables, i.e., the overall density is a product of several lower-dimensional functions.
Under the setting of unconditional density estimation (i.e., $X= \varnothing$ ($d_x=0$) in Algorithm \ref{algo:ede}), this low-dimensional structure $\mathcal{H}_{\mathsf{f}}(d, d^\star, \beta, C)$ characterizes the function form of many graphical models of interest, e.g., Markov random field (MRF) and Bayesian network. 
MRF represents the joint distribution of a set of random variables using an undirected graph, where edges encode conditional dependencies. 
To be specific, consider a d-dimensional random vector $Z = (Z_1, \dots, Z_d)$, and associate $Z$ with a graph $G = (V, E)$, where $V=\{1,\ldots, d\}$ consists of $d$ vertices, each corresponding to a coordinate $Z_j$ with $j\in [d]=V$, and $E$ represents the set of edges. We say a density $p(z_1,\ldots, z_d)$ satisfies the Markov property with respect to a graph $G$ if $Z_{A} \indep_{p} Z_{B} | Z_{C}$ for any vertices index $C\subseteq V$ such that deleting the vertices and corresponding edges in $C$ 
breaks the graph $G$ into two disconnected components $A, B$. The well-known Hammersley–Clifford theorem connects the Markov property and factorizable function form of $p$ when $p$ is strictly positive on its domain: 
\begin{proposition}[Hammersley–Clifford theorem]
\label{eq:}
Suppose $p(z_1,\ldots, z_d)$ is a strictly positive density on its domain, then the following two statements are equivalent: (1) $p$ satisfies the Markov property with respect to $G$; and (2)
\begin{align*}
    p(z) = \prod_{J \in \mathcal{C}(G)} f_J(x_J)
\end{align*} where $\mathcal{C}(G)$ is the set of all the cliques of $G$ defined formally as $\mathcal{C}(G) = \{V'\subseteq V: (i,j)\in E ~\text{for any}~ i,j\in V',i\neq j\}$.
\end{proposition}
It is immediate from the above theorem that any density function $p(y)$ satisfying the Markov property with graph $G$ satisfies $p\in \mathcal{H}_{\mathsf{f}}(d, d^\star, \beta, C)$ where $d^*$ is the size of the maximum clique, as soon as $f_J$'s are $\beta$-H\"{o}lder.   
Recently, a few papers have established convergence rates in TV distance when $p(y) \in \mathcal{H}_{\mathsf{f}}(d, d^\star, \beta, C)$ under the unconditional density estimation setup. 
For example, \cite{bos2023supervised} proposed a two-stage explicit density estimator using neural networks and establishes the convergence rate $n^{-\beta/(2\beta+d^\star)} + n^{-\alpha/d}$ where $\alpha$ is the H{\"o}lder smoothness parameter of the whole function $p$; this rate cannot circumvent the curse of dimensionality in general, given $\alpha=\beta$ without further assumptions. 
\cite{vandermeulen2024dimension} proposed a neural network-based least square estimator that can achieve the rate $n^{-\beta/(4\beta+d^\star)}$ when $\beta=1$, but clearly this rate is not minimax optimal. 
While \cite{kwon2025nonparametric} constructs a rate-optimal implicit density estimator of order $n^{-\beta/(2\beta + d^\star)}$ using a diffusion model, their neural network implementation is computationally intractable.
As a comparison, we next argue that applying our \cref{thm:main} and \cref{thm:main2} can yield the optimal rate for conditional density estimation when $p_0(y|x) \in \mathcal{H}_{\mathsf{f}}(d, d^\star, \beta, C)$ with $d=d_x + d_y$, (it admits the unconditional density estimation as a special case). 
Before stating our main result, we first present a condition specifying the choice of relevant hyperparameters: 

\begin{condition}
\label{cond:interaction}
    We adopt the following hyperparameter configurations. 
    \begin{itemize}[itemsep=0pt]
        \item[(a)] For the explicit density estimation, we choose the depth $L$ and width $N$ of the neural network such that $LN \asymp n^{\frac{d^\star}{2(2\beta+d^\star)}}$.
        \item[(b)] For estimation the score function in implicit density estimation, we use  $K \gtrsim (NL)^2$ Monte Carlo samples at each step of the backward diffusion process, along with the early stopping parameter $\delta\asymp (NL)^{-2}$, truncation parameter $T\asymp \log(n)$ parameters, and the discretization hyperparameter $M\asymp (NL)^2$.
    \end{itemize}
\end{condition}

\begin{corollary}
\label{cor:lowdim_factor}
    Under the setting of \cref{thm:main}, suppose further $p_0(y|x) \in \mathcal{H}_{\mathsf{f}}(d_y+d_x, d^\star, \beta, C)$. With neural network hyper-parameter choice in \cref{cond:interaction} (a), the explicit neural density estimator $\hat{p}$ in \cref{algo:ede} satisfies
    \begin{align*}
         \mathsf{R}_{\mathsf{TV}}(p_0, \hat{p}) + \sqrt{\mathsf{R}_{\mathsf{D}_f}(p_0, \hat{p})}= \tilde{O}(n^{-\frac{\beta}{2\beta+d^\star}})
    \end{align*} with probability at least $1-n^{-100}$, where the randomness is taken over the i.i.d. samples $\{(X_i, Y_i, \tilde{Y}_i)\}_{i=1}^n$. With the backward diffusion process hyperparameter choices in \cref{cond:interaction} (b), the distribution $p_{\hat{Y}|X}$ of the samples generated by the implicit neural density estimator in \cref{algo:ide} satisfies
    \begin{align*}
         \mathsf{R}_{\mathsf{TV}}(p_0, p_{\hat{Y}|X}) + \sqrt{\mathsf{R}_{\mathsf{KL}}(p_0, p_{\hat{Y}|X})}= \tilde{O}(n^{-\frac{\beta}{2\beta+d^\star}}).
    \end{align*} Here $\tilde{O}(\cdot)$ absorbs the constants $(d_y, d_x, d^\star, \beta, C, c_1, c_2)$ and $\mathrm{poly}(\log(n))$ factors.
\end{corollary}
It is immediately evident from the above Corollary that CINDES can achieve a minimax optimal rate (up to a log factor) \emph{without knowing the graph explicitly}; all we need to know is $(\beta, d^*)$ (or any upper bound thereof). 
It is also possible to adapt to the unknown parameters $(\beta, d^*)$ by employing a truncated $\ell_1$-norm penalty on the weights of the neural network (see \cite{fan2024factor} for details). However, we choose not to pursue this direction in order to maintain the clarity of exposition.

\subsection{Low-dimensional compositional structures}

Neural networks are known for their capability to be adaptive to the low-dimensional composition structures both empirically \citep{sclocchi2025phase} and theoretically in regression tasks \citep{fan2024factor, kohler2021rate, schmidt2020nonparametric}. 
In this section, we extend that result to the implicit and the explicit density estimation. 
Towards that goal, we first introduce the definition of hierarchical composition models: 

\begin{definition}[Hierarchical composition model $\mathcal{H}_{\mathsf{hcm}}(d, l, \mathcal{O}, C)$]
\label{hcm}
    We define function class of hierarchical composition model $\mathcal{H}_{\mathsf{hcm}}(d, l, \mathcal{O}, C)$ \citep{kohler2021rate} with $l, d \in \mathbb{N}^+$, $C\in \mathbb{R}^+$, and $\mathcal{O}$, a subset of $[1,\infty) \times \mathbb{N}^+$, in a recursive way as follows. Let $\mathcal{H}_{\mathsf{h}}(d, 0,\mathcal{O}, C)=\{h(x)=x_j, j\in [d]\}$, and for each $l\ge 1$, 
    \begin{align*}
    \mathcal{H}_{\mathsf{hcm}}(d, l,\mathcal{O}, C) = \big\{&h: \mathbb{R}^d \to \mathbb{R}: h(x) = g(f_1(x),...,f_t(x))\text{, where} \\
    &~~~~~ g\in \mathcal{H}_{\mathsf{h}}(t, \beta, C) \text{ with } (\beta, t)\in \mathcal{O} \text{ and } f_i \in \mathcal{H}_{\mathsf{hcm}}(d, l-1,\mathcal{O}, C)\big\}.
\end{align*} 
\end{definition}
Basically, $\mathcal{H}_{\mathsf{hcm}}(d, l, \mathcal{O}, C)$ consists of all functions that are composed $l$ times of functions of $t$ dimensions with smoothness $\beta$ with $(\beta, t)\in \mathcal{O}$.
Here, we assume that all components are at least Lipschitz functions to simplify the presentation, as in \cite{kohler2021rate}. For standard regression task, the minimax optimal $L_2$ estimation risk over $\mathcal{H}(d, l, \mathcal{O}, C_h)$ is $n^{-\alpha^\star/(2\alpha^\star + 1)}$, where $\alpha^\star = \min_{(\beta, t)\in \mathcal{O}} (\beta/t)$ is the smallest dimensionality-adjusted degree of smoothness \citep{fan2024factor} that represents the hardest component in the composition. The hierarchical composition model also admits the factorizable structure $\mathcal{H}_{\mathsf{f}}(d, d^\star, \beta, C)$ as special cases with $\alpha^\star = \beta/d^\star$, yet includes more functions with intrinsic low-dimensional structures. For example, if $f(x) = f_1(x_2, x_6) \cdot f_2(f_3(x_2, x_3), f_4(x_4, x_5)) \cdot f_5(x_1, x_3, x_5)$ and all functions have a bounded second derivative $\beta=2$, then the hardest component is the last one, and the dimensionality-adjusted degree of smoothness is $\alpha^* = 2/3$ rather than $\beta/d = 2/6=1/3$ or $\beta/d^\star=2/4=1/2$.

With the choice of the hyperparameters in \cref{cond:hcm}, our CINDES estimator can also achieve an optimal rate when the conditional density functions lie within the hierarchical composition model defined in \cref{hcm}. 
\begin{condition}
\label{cond:hcm}
    We adopt the following hyperparameter configurations. 
    \begin{itemize}[itemsep=0pt]
        \item[(a)] The neural network depth $L$ and width $N$ satisfying $LN \asymp n^{\frac{1}{2(2\alpha^\star + 1)}}$.
        \item[(b)] For estimation of the score function in implicit density estimation, we use the same hyperparameter setting as in Condition \ref{cond:interaction} with the choice of $NL$ mentioned in (a). 
    \end{itemize}
\end{condition}

\begin{corollary}
\label{cor:lowdim_compose}
    Under the setting of \cref{thm:main}, suppose further $p_0(y|x) \in \mathcal{H}_{\mathsf{hcm}}(d, l, \mathcal{O}, C)$. With neural network hyper-parameter choice in \cref{cond:hcm} (a), the explicit neural density estimator $\hat{p}$ in \cref{algo:ede} satisfies
    \begin{align*}
         \mathsf{R}_{\mathsf{TV}}(p_0, \hat{p}) + \sqrt{\mathsf{R}_{\mathsf{D}_f}(p_0, \hat{p})}= \tilde{O}(n^{-\frac{\alpha^\star}{2\alpha^\star+1}})
    \end{align*} with probability at least $1-n^{-100}$, where the randomness is taken over the i.i.d. samples $\{(X_i, Y_i, \tilde{Y}_i)\}_{i=1}^n$. With the backward diffusion process hyperparameter choices in \cref{cond:hcm} (b), the distribution $p_{\hat{Y}|X}$ of the samples generated by the implicit neural density estimator in \cref{algo:ide} satisfies
    \begin{align*}
         \mathsf{R}_{\mathsf{TV}}(p_0, p_{\hat{Y}|X}) + \sqrt{\mathsf{R}_{\mathsf{KL}}(p_0, p_{\hat{Y}|X})}= \tilde{O}(n^{-\frac{\alpha^\star}{2\alpha^\star+1}}).
    \end{align*} Here $\tilde{O}(\cdot)$ absorbs the constants $(l, \sup_{\beta, t\in \mathcal{O}} (\beta \lor t), C, d_y, d_x, c_1, c_2)$ and $\mathrm{poly}(\log(n))$ factors.
\end{corollary}

\section{Simulation Studies}
\label{sec:simulation}
In this section, we evaluate the empirical performance of CINDES against the following three competing methods (if applicable). 
\begin{itemize}
    \item[(1)] Random Forest Classifer Density Estimator (RFCDE): It shares a similar idea with our CINDES estimator, where the machine learning module is replaced by a random forest (instead of a neural network). It is applicable to all the density estimation tasks.
    \item[(2)] Masked Autoregressive Flow (MAF) \citep{papamakarios2017masked}: It is a neural density estimator in the family of normalizing flows, where the target distribution of the response is modeled as a base measure pushed forward by a series of invertible transforms that are parameterized by neural networks. 
    \item[(3)] LinCDE \citep{gao2022LinCDE}: The estimator uses tree boosting and Lindsey's method to estimate the conditional density, but is only applicable for univariate response $Y$.
\end{itemize}

\noindent{\textbf{Implementation.}} For our estimator and the RFCDE, the ``fake samples" $\tilde Y_1, \dots, \tilde Y_n \in \mathbb{R}^{d_y}$ are sampled uniformly from $\hat{\mathcal{Y}} = \prod_{j=1}^{d_y}[\min_{i} Y_{i,j}, \max_{i} Y_{i,j}]$. As for the neural network architecture, we adopt a fully connected neural network with depth 3, and width 64 for our CINDES estimator. For MAF, we employ a 3-layer architecture of a normalizing flow model with 2 sequential transformations, each implemented by a masked autoregressive layer with 64 hidden features and a standard Gaussian as a base density. The weights for both neural network estimators are optimized using the Adam optimizer with a learning rate of $10^{-3}$, $L_2$ regularization with a hyper-parameter picked from $\{10^{-3}, 5 \times 10^{-4}, 2\times 10^{-4}, 10^{-4}, 0\}$ and early stopping using another validation set. For model selection, we pick the model that minimizes the negative log-likelihood (NLL) $\sum_{(x,y) \in \mathcal{D}_{valid}} \log \hat{p}(y|x)$ using validation set $|\mathcal{D}_{valid}| = 0.25 |\mathcal{D}_{train}| = 0.25 n$. For RFCDE, we use a random forest with 200 trees, where each tree has a maximum depth of 12. The LinCDE for conditional density estimation on univariate response uses default hyperparameters.
We evaluate different estimators using the empirical TV distance between the estimated density and the ground-truth density under $(x,y) \in \mathcal{D}_{test} \overset{i.i.d.}{\sim} \mu_{0, x} \times \mathrm{Uniform}(\mathcal{Y})$ by
$$
\hat{\mathsf{TV}}(\hat{p},p_0) = \frac{1}{|\mathcal{D}_{test}|}\sum_{(x,y)\in \mathcal{D}_{test}}|\hat{p}(y|x) - p_0(y|x)|
$$
Especially, for unconditional density estimation $X =  \varnothing$, $\hat{\mathsf{TV}}(\hat{p},p_0) = \frac{1}{|\mathcal{D}_{test}|}\sum_{y\in \mathcal{D}_{test}}|\hat{p}(y) - p_0(y)|$.
Section~\ref{sec:uncond} presents simulations for unconditional density estimation, and Section~\ref{sec:cond} presents simulations for conditional density estimation.

\subsection{Unconditional density estimation}
\label{sec:uncond}
In this section, we empirically compare the performance of CINDES for unconditional density estimation (as outlined in Algorithm \ref{algo:ede} with $X =  \varnothing$) with RFCDE and MAF. 

\medskip
\noindent \textbf{Data Generating Process.} We consider the following two bivariate distributions. The ground-truth density function is visualized in the first column of \cref{fig:unconditional}.
\begin{enumerate}
    \item [(a)] \underline{Spherical Gaussian mixture.} In this case, the observations are generated from a mixture of 6 Gaussian distributions $Y \sim \frac16 \sum_{j = 1}^6 \cN(\mu_j, 0.01 \mathbf{I}_2)$ with $\mu_j = \left(\frac12 \cos\left(\frac{2\pi j}{6}\right), \frac12\sin\left(\frac{2\pi j}{6}\right)\right) $.
    \item [(b)] \underline{Elliptical Gaussian mixture.} The data-generating process is similar to the previous setup, but we choose $Y \sim \frac18 \sum_{j = 1}^8 \cN(\mu_j, \boldsymbol{\Sigma}_j)$ where $\mu_j = \left(3 \cos\left(\frac{\pi j}{4}\right), 3\sin\left(\frac{\pi j}{4}\right)\right)$ and $$\mathbf{\Sigma}_j = \left[\begin{array}{cc}\cos ^2 \frac{\pi j}{4}+0.16^2 \sin ^2 \frac{\pi j}{4} & \left(1-0.16^2\right) \sin \frac{\pi i}{4} \cos \frac{\pi i}{4} \\ \left(1-0.16^2\right) \sin \frac{\pi j}{4} \cos \frac{\pi j}{4} & \sin ^2 \frac{\pi j}{4}+0.16^2 \cos ^2 \frac{\pi j}{4}\end{array}\right].$$
\end{enumerate}
Different estimators observe $Y_1, \dots, Y_n \overset{i.i.d.}{\sim} F$ where $n = 12000$ and $F$ varies among the two choices mentioned above.

\noindent \textbf{Results.} The TV distances of the procedures are presented in Table \ref{tab:uncond}, where our method has significantly smaller TV distance than other methods. We further assess all density estimates over $\mathcal{Y}$, discretized into a $100 \times 100$ evaluation test grid in Figure \ref{fig:unconditional}. 
Each row of the figure represents one of the two data-generating distributions, and each column presents a method for estimating density (with the left-most column being the true density).
It is immediate from the figure that for six mixture Gaussian distribution (both with constant variance (a) and non-constant variance (b)), CINDES cleanly recovers the mixture components without any spurious ``bridges'' between them; other estimators either produce blocky, grid‐like artifacts with residual noise between clusters (RFCDE) or overly smooth, unrealistic connections linking separate modes (MAF). 

\begin{figure}[ht]
  \centering
     \subfigure[Spherical Gaussian mixture]{
         \centering
         \includegraphics[width = \textwidth]{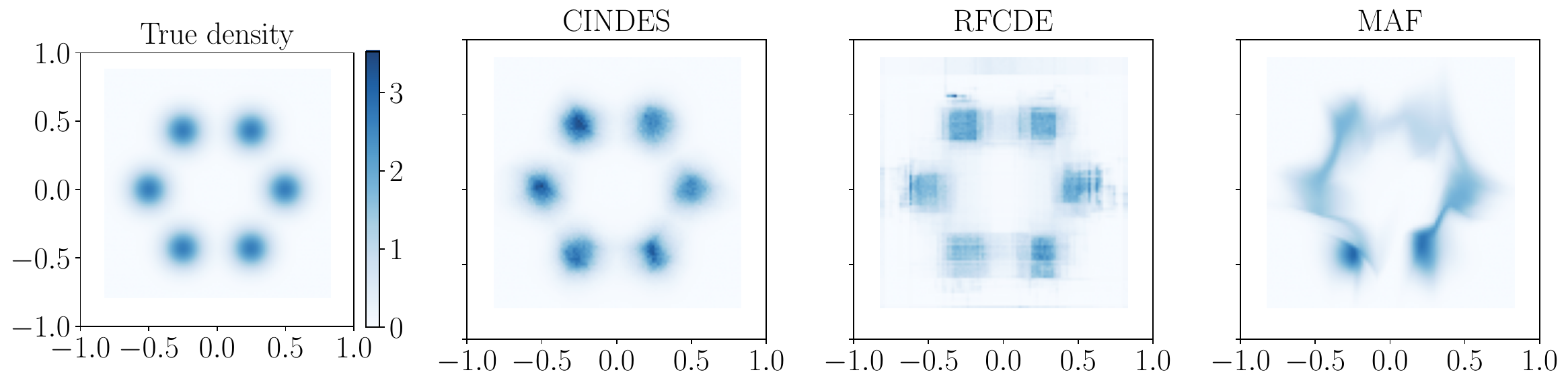}
     }
     \subfigure[Elliptical Gaussian mixture]{
         \centering
         \includegraphics[width = \textwidth]{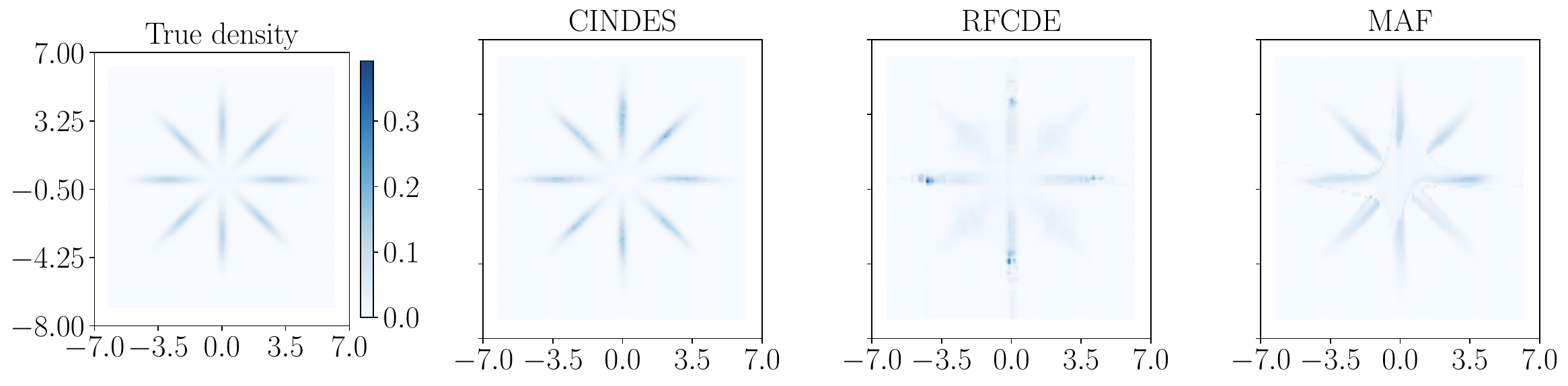}
     }
     \caption{True density and estimated density by different density estimators (CINDES, RFCDE, and MAF) in one trial for two data-generating processes. Density plots were shown on a $100 \times 100$ grid 2D bounded region. Ground-truth densities were shown in the first column. Each row plots results of one data-generating process: (a) Spherical Gaussian mixture; (b) Elliptical Gaussian mixture.} 
     \label{fig:unconditional}
\end{figure}  

\begin{table}[htbp]
\centering \small
\begin{tabular}{lccc}
\toprule
 & CINDES & RFCDE & MAF \\
\midrule
Spherical Gaussian Mixture & $\textbf{0.0475} \pm 0.0040 $ & $0.1282\pm 0.0074$ & $0.1323 \pm 0.0098$\\
Elliptical Gaussian Mixture & $\textbf{0.0011} \pm 0.0000$ &$0.0047 \pm 0.0003$ &$0.0027 \pm 0.0003$
 \\
\bottomrule
\end{tabular}
\caption{Empirical TV distance for unconditional density estimation under 100 replications. Lower values indicate better performance.}
\label{tab:uncond}
\end{table}

\subsection{Conditional density estimation}
\label{sec:cond}
In this section, we present our simulation results for conditional density estimation. We consider two different scenarios: i) when the conditioning variable $X$ is multivariate but the response variable $Y$ is univariate, and ii) when both $X$ and $Y$ are multivariate. 

\subsubsection{Univariate $Y$ and Multivariate $X$}
\noindent \textbf{Data generating process.}
The responses $Y \in \mathbb{R}$ given the covariates $X \in \mathbb{R}^{d_x}$ simulated from Uniform $(\mathcal{X})$ are generated from the following distribution. We generate $(X_i, Y_i)_{i=1}^n$ with different training sample sizes $n = 500, 2000, 8000$, validating dataset with 1/4 corresponding training sample size and test data $(x,y)\in \mathcal{D}_{test} \overset{i.i.d}{\sim}\mu_{0,x} \times \mathrm{Uniform}(\mathcal{Y})$ with $|\mathcal{D}_{test}| = 500^2$.
\begin{enumerate}
\item [(a)] \underline{Nonlinear model.} 
In this case, the covariate $X$ is generated uniformly from $\mathcal{X} = [-1, 1]^{4}$ and the response is generated as $Y|X = x \sim p(y|x) = \frac{1}{2}\left(1-y\tanh\left(\sin(x_1) + x_2^2 - \frac{1}{2}x_3\right)\right)$ with support $Y \subset[-1,1]$ .
    
\item [(b)] \underline{Additive model.} 
In this case, $X$ is generated uniformly from $\mathcal{X} =[0,1]^{d_x}$ (with $d_x = 20$) and the response $Y$ is generated as $Y|X = x \sim\mathcal{TN}_1(\mu(x), 2^2)$ with $\mu(x) = \sum_{j=1}^5 \mu_j(x_j)$, where $\mathcal{TN}_M(\mu,\sigma^2)$ is the truncated normal distribution within the interval $[-M, M]$. 
The mean functions $\{\mu_j(x)\}_{j = 1}^5$ are randomly selected from the set of univariate functions $\{\cos (\pi x), \sin(x), (1-|x|)^2, (1+e^{-x})^{-1}, 2\sqrt{|x|}-1\}$. 
    
\item [(c)] \underline{Gaussian mixture.} 
In this case, the covariate $X$ is generated uniformly from $\mathcal{X} = [0, 1]^{d_x}$ (with $d_x = 4$) and the response is generated from a mixture normal distribution as: $Y|X = x \sim (1-\pi(x))\mathcal{TN}_{0.85}(\mu_1(x), 0.15^2) + \pi(x)\mathcal{TN}_{0.85}(\mu_2(x), 0.12^2)$, where $\pi(x) = (1+\exp(-0.2-1.2x_1+0.8x_2-0.6x_3+0.4x_4))^{-1}$, $\mu_1(x) = 0.6x_1-0.3x_2+0.2x_3+0.4\sin(2\pi x_1)+0.2\cos(2\pi x_2)$ and $\mu_2(x) =-0.5 x_1+0.2 x_2-0.25 x_3+0.1 x_4-0.35 \sin \left(2 \pi x_1\right)+0.25 \cos \left(2 \pi x_3\right)$.
\end{enumerate}

\noindent{\bf Results.}
Our simulation results are presented in Table \ref{tab:tv_results_univariate_y}. We report the average TV distance between the estimated density and the ground-truth density, averaged over $100$ Monte Carlo replications. The results in Table \ref{tab:tv_results_univariate_y} show that CINDES consistently achieves smaller estimation error than RFCDE, MAF, and LinCDE across all sample sizes, demonstrating the efficacy of our proposed method. 

\begin{table}[htbp]
\centering \small
\begin{tabular}{cccccc}
\toprule
Experiment & Sample size &CINDES  & RFCDE & MAF & LinCDE \\
\midrule
I(a) &500 & $\textbf{0.0938} \pm 0.027$ & $0.1668 \pm 0.018$ & $0.1763 \pm 0.017$ & $0.1029 \pm 0.016$ \\
&2000 & $\textbf{0.0665} \pm 0.011$ & $0.1173 \pm 0.010$ & $0.1626 \pm 0.011$ & $0.0767 \pm 0.008$ \\
&8000 & $\textbf{0.0473} \pm 0.007$ & $0.0812 \pm 0.004$ & $0.1576 \pm 0.007$ & $0.0559 \pm 0.004$ \\
\hline
I(b) &500 & $\textbf{0.0677} \pm 0.026$ & $0.0907 \pm 0.012$ & $0.1709 \pm 0.021$ & $0.1014 \pm 0.013$ \\
&2000 & $\textbf{0.0550} \pm 0.014$ & $0.0701 \pm 0.006$ & $0.1609 \pm 0.012$ & $0.0661 \pm 0.009$ \\
&8000 & $\textbf{0.0418} \pm 0.008$ & $0.0520 \pm 0.005$ & $0.1567 \pm 0.007$ & $0.0470 \pm 0.005$ \\
    \hline
I(c) &500 & $0.3779 \pm 0.025$ & $0.3859 \pm 0.015$ & $0.3632 \pm 0.019$ & $\textbf{0.3503} \pm 0.012$ \\
&2000 & $\textbf{0.2609} \pm 0.019$ & $0.3282 \pm 0.012$ & $0.2827 \pm 0.012$ & $0.3233 \pm 0.009$ \\
&8000 & $\textbf{0.1684} \pm 0.009$ & $0.2879 \pm 0.009$ & $0.2549 \pm 0.011$ & $0.3190 \pm 0.009$ \\
\bottomrule
\end{tabular}
\caption{Empirical TV distance for each estimator across different experiments and different training sample sizes under 100 replications. Lower values indicate better performance.}
\label{tab:tv_results_univariate_y}
\end{table}

\subsubsection{Multivariate $X$ and multivariate $Y$}
In this subsection, we explore the situation when we have a multivariate covariate $X$, a multivariate response $Y$, and we aim to estimate the conditional density of $Y$ given $X$. 
\\\\
\noindent
{\bf Data generating process.}
Here we generate the covariate/conditioning variable $X$ uniformly from $[0, 1]^{d_x}$ with $d_x  = 16$. The response variable $Y \in \reals^{d_y}$, with $d_y = 4$, is generate as $Y \mid X = x \sim \mathcal{TN}_1(Wx, I_{d_y})$.
Here the matrix $W \in \reals^{d_y \times d_x}$. 
Each row of $W$ is generated from a Dirichlet distribution with $\mathbf{\alpha} = \mathbf{1}_{d_x}$ and kept fixed throughout the experiment.

\begin{table}[htbp]
\centering 
\begin{tabular}{cccc}
\toprule
Sample size &CINDES  & RFCDE & MAF  \\
\midrule
500 & $\textbf{0.0161} \pm 0.002$ & $0.0190 \pm 0.002$ & $0.0313 \pm 0.002$ \\
2000 & $\textbf{0.0144} \pm 0.001$ & $0.0148 \pm 0.001$ & $0.0301 \pm 0.001$ \\
8000 & $\textbf{0.0105} \pm 0.002$ & $0.0113 \pm 0.001$ & $0.0245 \pm 0.003$ \\\bottomrule
\end{tabular}
\caption{Empirical TV distance for each estimator and different training sample sizes under 100 replications. Lower values indicate better performance.}
\label{tab:tv_results_multivariate_y}
\end{table}

\noindent
{\bf Results. } 
Our results are summarized and presented in Table \ref{tab:tv_results_multivariate_y}. As before, we compare our method against RFCDE and MAF. However, we exclude LinCDE from this comparison as it is not practically applicable to settings with a multivariate response. The method's reliance on a basis function expansion becomes computationally intractable as the response dimension grows. As in the case of the univariate response in the previous subsection, we vary the sample size $n \in \{500, 2000, 8000\}$, and report the TV distance between the estimated density and the true density averaged over 100 Monte Carlo repetitions. Here also, CINDES yields a smaller estimator error across all sample sizes compared to other methods, which establishes the efficacy of our proposed methodology.

\section{Real data analysis}
\label{sec:real_data}
In this section, we showcase the performance of our methodology on estimating conditional density using the light tunnel dataset of \cite{gamella2025chamber}. 
The light tunnel is a physical chamber with a controllable light source at one end and two linear polarizers mounted on rotating frames. Sensors measure the light intensity before, between, and after the polarizers. 
Specifically, we aim to estimate the joint and/or conditional probability densities of the system's key variables: the angles of the two polarizers and the light intensities recorded by the sensors.
The variables we consider in the experiment are $(R, G, B, \tilde{C}, \theta_1, \theta_2, \tilde{I}_1, \tilde{V}_1, \tilde{I}_2, \tilde{V}_2, \tilde{I}_3, \tilde{V}_3)$, where $(R, G, B)$ is the brightness of the red, green and blue LEDs on the main light source, $\tilde{C}$ is the electric current drawn by the light source, $(\theta_1, \theta_2)$ are the angles of the polarizer frames, and $(\tilde{I}_1, \tilde{V}_1, \tilde{I}_2, \tilde{V}_2, \tilde{I}_3, \tilde{V}_3)$ represent the measurement of light-intensity sensors placed at different positions of polarizers. To make the scale homogeneous, we first standardize each variable by subtracting its mean and dividing it by its standard deviation. 
Now we consider discussing the conditional density of response variables given other variables. 
We divide the conditional density estimation into two categories: 

\noindent (a) \underline{Univariate conditional density estimation:} We pick one of the variables as the response, and the goal is to estimate the conditional density of the selected variable given other variables.

\noindent (b) \underline{Multivariate conditional density estimation: } We consider a group of the variables as a multivariate response, and the goal is to estimate the conditional density of this group of response variables given other variables. The groups that we consider here as response variable are $(R, G, B, \tilde{C})$, $(\theta_1, \theta_2)$, $(\tilde{I}_1, \tilde{V}_1, \tilde{I}_2, \tilde{V}_2)$, and $(\tilde{I}_3, \tilde{V}_3)$ given their semantic similarity.  

For our estimator, we use a neural network with depth $L = 3$ and width = $64$. We use the $L_2$ penalty as a regularization technique with early stopping and select the best model determined by the validation set. 
In order to make sure that the support of the fake responses $\tilde Y$ contains the support of the true responses $Y$, we here generate $\tilde Y \sim \cN(\mu, \Sigma)$, where $\mu$ and $\Sigma$ are the estimated mean and variance of the true response $Y$. 
Consequently, our Algorithm \ref{algo:ede} needs to be modified slightly: instead of setting $\hat p(y \mid x) = \exp(\hat f(y \mid x))$, we set it as $\hat p(y \mid x) =  \exp(\hat f(y \mid x)) \phi(y; \mu, \Sigma)$, where $\phi(y; \mu, \Sigma)$ is the Gaussian density with mean $\mu$, variance $\Sigma$, evaluated at $y$. 
As before, we compare our method with MAF, RFCDE, LinCDE, and introduce a new competitor, LocScale-NN (location-scale neural network), which models the conditional density of $Y$ given $X$ as $\cN(\mu(X), \Sigma(X))$ and estimates $\mu(\cdot), \Sigma(\cdot)$ using neural networks. For LocScale-NN, we adopt the same structure as the neural network in our estimators. 
\begin{table}[htbp]
\centering \footnotesize
\begin{tabular}{c c c c c c}
\hline
Response & Loc-ScaleNN & CINDES & RFCDE & MAF & LinCDE \\
\hline
\texttt{red}     & 0.645 $\pm$ 0.470 & \textbf{-0.410} $\pm$ 0.809 & 0.669 $\pm$ 0.022 & -0.259 $\pm$ 0.050 & 0.686 $\pm$ 0.017 \\
\texttt{green}   & 0.537 $\pm$ 0.200 & \textbf{0.134} $\pm$ 0.777 & 0.897 $\pm$ 0.017 & 0.243 $\pm$ 0.052 & 0.998 $\pm$ 0.014 \\
\texttt{blue}    & 0.693 $\pm$ 0.536 & \textbf{-0.671} $\pm$ 0.556 & 0.631 $\pm$ 0.024 & -0.295 $\pm$ 0.057 & 0.633 $\pm$ 0.014 \\
\texttt{current} & 0.970 $\pm$ 0.044 & \textbf{0.868} $\pm$ 0.699 & 0.956 $\pm$ 0.031 & 0.936 $\pm$ 0.038 & 0.969 $\pm$ 0.038 \\
\texttt{pol\_1}  & 1.361 $\pm$ 0.029 & 1.159 $\pm$ 0.365 & \textbf{1.156} $\pm$ 0.013 & 1.369 $\pm$ 0.032 & 1.257 $\pm$ 0.010 \\
\texttt{pol\_2}  & 1.362 $\pm$ 0.031 & 1.160 $\pm$ 0.358 & \textbf{1.146} $\pm$ 0.014 & 1.361 $\pm$ 0.032 & 1.247 $\pm$ 0.015 \\
\texttt{ir\_1}   & 0.532 $\pm$ 0.462 & \textbf{-0.772} $\pm$ 0.596 & 0.645 $\pm$ 0.025 & -0.353 $\pm$ 0.057 & 0.493 $\pm$ 0.027 \\
\texttt{vis\_1}  & 0.502 $\pm$ 0.464 & \textbf{-0.649} $\pm$ 0.667 & 0.629 $\pm$ 0.025 & -0.403 $\pm$ 0.060 & 0.473 $\pm$ 0.023 \\
\texttt{ir\_2}   & 0.585 $\pm$ 0.508 & \textbf{-0.753} $\pm$ 0.602 & 0.627 $\pm$ 0.024 & -0.363 $\pm$ 0.052 & 0.491 $\pm$ 0.024 \\
\texttt{vis\_2}  & 0.449 $\pm$ 0.446 & \textbf{-0.754} $\pm$ 0.624 & 0.627 $\pm$ 0.026 & -0.440 $\pm$ 0.050 & 0.460 $\pm$ 0.029 \\
\texttt{ir\_3}   & 0.493 $\pm$ 0.416 & \textbf{-0.464} $\pm$ 0.714 & 0.666 $\pm$ 0.022 & -0.313 $\pm$ 0.047 & 0.540 $\pm$ 0.029 \\
\texttt{vis\_3}  & 0.413 $\pm$ 0.264 & -0.291 $\pm$ 0.815 & 0.653 $\pm$ 0.023 & \textbf{-0.320} $\pm$ 0.044 & 0.522 $\pm$ 0.032 \\
\hline
\end{tabular}
    \caption{Average NLL across models. Each row corresponds to a response variable; the lowest value per row is bolded.}
    \label{Tab:NLL for real data}
\end{table}

\begin{table}[bp]
    \centering \small
    \begin{tabular}{lcccc}
    \hline
    Response & LocScale-NN & CINDES & RFCDE & MAF \\
    \hline
    \texttt{red}, \texttt{green}, \texttt{blue}, \texttt{current} & $4.637 \pm 0.079$ & $\textbf{1.336} \pm 0.103$ & $5.240 \pm 0.233$ & $1.577 \pm 0.104$ \\
    \texttt{pol\_1}, \texttt{pol\_2} & $2.841 \pm 0.030$ & $2.534 \pm 0.040$ & $\textbf{2.318} \pm 0.033$ & $2.708 \pm 0.048$ \\
    \texttt{ir\_1}, \texttt{vis\_1}, \texttt{ir\_2}, \texttt{vis\_2} & $2.520 \pm 0.844$ & $\textbf{-2.288} \pm 0.365$ & $5.101 \pm 0.257$ & $-0.774 \pm 0.137$ \\
    \texttt{ir\_3}, \texttt{vis\_3} & $1.377 \pm 1.152$ & $-0.153 \pm 0.071$ & $3.038 \pm 0.357$ & $\textbf{-0.158} \pm 0.075$ \\
    \hline
    \end{tabular}
    \caption{Average NLL across models. Each row corresponds to a response variable; the lowest value per row is bolded.}
    \label{tab: NLL for multivariate}
\end{table}

We repeat the experiment 100 times. In each trial, a randomly selected subset of 3000 data is used for all the estimators. Among these selected data, we use 40\% as training data, 10\% as validation data, and 50\% as test data and evaluate the performance of estimators via the normalized NLL, defined as: 
$$
\textsf{NLL}(\hat p) = \frac{1}{n_{test}}\sum_{(x,y)\in \mathcal{D}_{test}} \left[\log \hat{p}(y|x) - \log \left( \frac{\mathrm{Vol}(\hat{\mathcal{Y}})}{n_{\tilde{y}}}\sum_{i = 1}^{n_{\tilde y}}\hat{p}(\tilde{y_i}|x)\right) \right], 
$$
where $\tilde y_1, \dots, \tilde y_{n_{\tilde y}} \overset{i.i.d.}{\sim} \mathrm{Uniform}(\hat{\mathcal{Y}})$. 
The results for the univariate conditional density estimation are presented in Table \ref{Tab:NLL for real data}, and the results for multivariate conditional density estimation are presented in Table \ref{tab: NLL for multivariate}. 
It is immediately from the tables that CINDES outperforms other methods, consistently achieving a smaller negative log-likelihood across different experiments.



\bibliographystyle{apalike2}
\bibliography{main}
\newpage

\appendix
\newpage
\section{Proofs for Explicit Density Estimator}
\label{sec:proof-explicit}

\subsection{A More General Result}
\label{sec:proof-explicit-main}

In this section, we present a more general result of \cref{thm:main}: it applies to any machine learning model $\mathcal{G}$ used, and the stochastic error is characterized by the critical radius of the local Rademacher complexity of the function class $\mathcal{G}$.

We first introduce the definition of the local Rademacher complexity, and the setting for a general machine learning model $\mathcal{G}$. Following the notations in the main text, recall that $d_x$ is the dimension of the covariate and $d_y$ is the dimension of the response; let $d = d_x + d_y$. For the function class $\mathcal{H} \subseteq \{h: \mathcal{X} \times \mathcal{Y} \to \mathbb{R}\}$, we define the localized Population Rademacher Compleixty as follows.

\begin{definition}[Localized Population Rademacher Complexity]
For a given radius $\delta>0$, function class $\mathcal{H}$, and distribution $\nu$ on $\mathcal{X}\times \mathcal{Y}$, define
\begin{align*}
    \mathsf{Rade}_{n,\nu}(\delta;\mathcal{H}) = \mathbb{E}_{Z,\varepsilon}\left[\sup_{h\in \mathcal{H}, \|h\|_{L_2(\nu)} \le \delta} \left|\frac{1}{n} \sum_{i=1}^n \varepsilon_i h(Z_i)\right|\right],
\end{align*} where $Z_1,\ldots, Z_n$ are i.i.d. samples from distribution $\nu$, and $\varepsilon_1,\ldots, \varepsilon_n$ are i.i.d. Rademacher variables taking values in $\{-1,+1\}$ with equal probability which are also independent of $(Z_1,\ldots, Z_n)$.
\end{definition} 

Let $\mathcal{G}$ be a class of functions defined on $\mathcal{X}\times\mathcal{Y} = \mathcal{X} \times [0,1]^{d_y}$, a subset of $\mathbb{R}^d$, we will establish the $L_2$ error between the ground-truth conditional density function $p_0$ and the following estimator $\hat{p}$ defined as 
\begin{align}
\label{eq:esta}
    \hat{p}(y|x) = \exp(\hat{g}(y,x)) ~~~ &\text{where} ~~~ \hat{g} \in \argmin_{g\in \mathcal{G} } \hat{\mathsf{L}}(g),
\end{align} where the empirical loss $\hat{\mathsf{L}}(g)$ is defined in \eqref{eq:loss}. The NCDE-NN estimator is a special case of the above procedure with $\mathcal{G} = \mathcal{H}_{\mathtt{nn}}(d_y + d_x, L, N, M)$. The following condition characterizes the uniform boundedness and statistical complexity of the machine learning model $\mathcal{G}$ we adopted.

\begin{condition}
\label{cond:func-class-g}
Letting $\nu_0$ be the joint distribution of $\mu_{0, x} \times \mathrm{Uniform}([0,1]^{d_y})$ and $\mu_0$  the joint distribution of $(X,Y)$, there exists a constant $c_3 \ge 1 \lor \log(c_1)$ such that the following conditions hold
\begin{itemize}[noitemsep]
\item[(1).] It is uniformly bounded by $c_3 \ge 1$, i.e., $\sup_{g\in \mathcal{G}} \|g\|_\infty \le c_3$.
\item[(2).] The critical radius of the local population Rademacher complexity for $\mathcal{G}$ is upper-bounded by $\delta_{n}$. In particular, for any $\nu\in \{\nu_0, \mu_0\}$, there exists some quantity $1/n \le \delta_n < 1$ such that
\begin{align*}
    \mathsf{Rade}_{n,\nu}(\delta; \partial \mathcal{G}) \le c_3 \delta_n \delta 
\end{align*} for any $\delta \in [\delta_n, 2c_3]$, where $\partial \mathcal{G} = \{g - g': g, g'\in \mathcal{G}\}$.
\end{itemize}
\end{condition}

We are ready to present a general result of \cref{thm:main}.

\begin{theorem}
\label{thm:main-ff}
Assume \cref{cond:regularity} and \cref{cond:func-class-g} hold, then for any $t>0$ and $n\ge 3$, the estimator \eqref{eq:esta} satisfies
\begin{align*}
    \| \hat{p} - p_0 \|_2 \le C \left\{\inf_{g\in \mathcal{G}} \|g - \log p_0\|_2 + \delta_{n} + \sqrt{\frac{t+\log(n)}{n}}\right\}
\end{align*} with probability at least $1-2e^{-t}$, where $C = \mathcal{O}(e^{4c_3})$.
\end{theorem}

\subsection{Proof of Theorem \ref{thm:main}}

We first introduce the notation of uniform covering number. Define $\|h\|_{\infty, X} = \sup_{x \in X} |h|$. Let $\mathcal{H}$ be a function class defined on $\mathcal{Z}$, we denote $\mathcal{N}(\epsilon, \mathcal{H}, d(\cdot, \cdot))$ to be the $\epsilon$-covering number of function class $\mathcal{H}$ with respect to the metric $d$, let
\begin{align*}
    &\mathcal{N}_p(\epsilon, \mathcal{H}, z_1^n) = \mathcal{N}\left(\epsilon, \mathcal{H}, d\right) \\
    &\qquad \text{with} \qquad d(f, g) = \begin{cases}
        \left(\frac{1}{n} \sum_{i=1}^n |f(z_i) - g(z_i)|^p\right)^{1/p} \qquad & 1\le p<\infty \\
        \max_{1\le i\le n} |f(z_i) - g(z_i)| \qquad & p = \infty
    \end{cases}
\end{align*} for any $p \in [1,\infty]$, and define the uniform covering number $\mathcal{N}_\infty(\varepsilon, \mathcal{H}, n)$ as
\begin{align*}
    \mathcal{N}_\infty(\epsilon, \mathcal{H}, n) = \sup_{z_1,\ldots, z_n} \mathcal{N}_{\infty}(\epsilon, \mathcal{H}, z_1^n)
\end{align*}

To prove \cref{thm:main} via applying \cref{thm:main-ff}, it suffices to verify \cref{cond:func-class-g} with $c_3 = \log(c_2)$. We will use the following technical lemma that applies to any generic uniformly bounded function class with a uniform covering number bound. 

\begin{lemma}[Calculating Local Rademacher Complexity with Uniform Covering Number, Lemma E.2 \cite{gu2024causality}]
\label{lemma:local-rademacher-uniform-covering-number}
    Let $Z_1,\ldots, Z_n \overset{i.i.d.}{\sim} \nu$ be random variables on $\mathcal{Z}$, and $\mathcal{H}$ be a function class satisfying $\sup_{h\in \mathcal{H}}\|h\|_\infty \le b$ \begin{align}
\label{eq:entropy-h}
    \log \mathcal{N}_\infty(\epsilon, \mathcal{H}, n) \le A_1 \log (A_2/\epsilon) \qquad \qquad \forall \epsilon \in (0, b]
\end{align} where $(A_1, A_2)$ are dependent on $\mathcal{H}$ and $n$ but independent of $\epsilon$. Then there exists some universal constant $C$ such that, for any $n\ge 3$
\begin{align*}
    \mathsf{Rade}_{n,\nu}(\delta; \mathcal{H}) \le b \delta_n \delta\qquad \forall \delta \in [\delta_n, b] 
\end{align*} with $\delta_n =C\sqrt{n^{-1} (A_1\log (A_2n)+\log(bn))}$.
\end{lemma}

\begin{proof}[Proof of \cref{thm:main}]
Applying further Theorem~7 of \cite{BHLM2019} yields the bound ${\rm Pdim}(\mathcal{G}) = {\rm Pdim}(\mathcal{H}_{\mathtt{nn}})  \lesssim W L \log(W)$, where $W$ is the number of parameters of the network $\mathcal{H}_{\mathtt{nn}}$. This indicates that
\begin{align*}
    {\rm Pdim}(\mathcal{G}) \lesssim (LN^2 + dN) L \log(LN^2 + dN) \lesssim L^2N^2 (1+\log n).
\end{align*}

Let $R = \log(c_2)$, it then follows from Theorem 12.2 of \cite{AB1999} that, for any $\epsilon \in (0,2 R]$
\begin{align*}
    \log \mathcal{N}_\infty(\varepsilon, \mathcal{G}, n)  &\le \left(\mathrm{Pdim}(\mathcal{G}) \right)\log\left(\frac{eRn}{\epsilon}\right) \\
    &\lesssim (NL)^2 (1+\log n) \log\left(eRn/\epsilon\right)
\end{align*}

Then it follows from \cref{lemma:local-rademacher-uniform-covering-number} that \cref{cond:func-class-g} is satisfied by setting
\begin{align*}
    c_3 = R = \log(c_2) \qquad \mathcal{G} = \mathcal{H}_{\mathtt{nn}}(d, L, N, R)
\end{align*}

It then concludes the proof by applying \cref{thm:main-ff}.

\end{proof}

\subsection{Proof of Theorem \ref{thm:main-ff}}

\newcommand{\vol}{{\mathsf{v}}}

We first introduce some notations. Let $\nu_0 = \mu_{0,x} \times \mathrm{Uniform}(\mathcal{Y})$, it worth noting that
\begin{align*}
    \|f\|_{2,\nu_0} = \sqrt{\int_{\mathcal{X}}  \int_{[0,1]^{d_y}} \left({|f(y, x)|^2} dy\right) \mu_{0,x}(dx)} =  \|f\|_2,
\end{align*} where the $L_2$ norm is defined on \eqref{eq:l2}.

Recall $\sigma(t) = 1/(1+e^{-t})$, we also define the population-level counterpart of the empirical loss \eqref{eq:loss}.
\begin{align}
    \label{eq:loss-popu}
    \mathsf{L}(g) = \mathbb{E}_{(X,Y)\sim \nu_0} \left[-p_0(Y|X)\log \sigma(g(Y, X)) - \log(1-\sigma(g(Y, X)))\right]
\end{align}

The first proposition establishes an approximate strong convexity around $\log p_0$. 
\begin{proposition}
Under \cref{cond:regularity} and \cref{cond:func-class-g}, we have
\label{prop:lb}
\begin{align*}
    \mathsf{L}(g) - \mathsf{L}(\tilde{g}) \ge \frac{1}{4e^{c_3}} \|g - \tilde{g}\|_2 ^2 - 4 c_1^2 e^{c_3} \|\tilde{g} - \log p_0\|_2^2.
\end{align*}
\end{proposition}
\begin{proof}[Proof of Proposition \ref{prop:lb}] See \cref{sec:proof-prop-lb}.
\end{proof}

Given any two functions $g, \tilde{g} \in \mathcal{G}$, define $\Delta(g, \tilde g)$ as: 
\begin{align*}
    \Delta(g, \tilde{g}) = \hat{\mathsf{L}}(g) - \hat{\mathsf{L}}(\tilde{g}) - \left(\mathsf{L}(g) - \mathsf{L}(\tilde{g})\right).
\end{align*}
The following proposition establishes an instance-dependent error bound on $\Delta(g, \tilde g)$. The error bound holds for any two functions $g$ and $\tilde{g}$, though we will pick $g$ to be the risk minimizer and $\tilde{g}$ as fixed in the proof of \cref{thm:main-ff}. The proof is relegated to \cref{proof:instance-dependent}.
\begin{proposition}
\label{prop:instance-dependent}
    Under \cref{cond:regularity} and \cref{cond:func-class-g}, for any $t>0$, denote $ \delta_{n, t} = \delta_n + \sqrt{\frac{t + 1 + \log(n c_3)}{n}}$, we have
    \begin{align*}
        \forall g, \tilde{g} \in \mathcal{G}, \qquad |\Delta(g, \tilde{g})| \le C \cdot c_3\left(\delta_{n, t}^2 + \delta_{n, t} \|g - {\tilde{g}}\|_{2}\right)
    \end{align*} occurs with probability at least $1-2e^{-t}$, where $C$ is a universal constant. 
\end{proposition}

Now we are ready to prove \cref{thm:main-ff}.

\begin{proof}[Proof of \cref{thm:main-ff}]
We use the fact that for any $\tilde{g} \in \mathcal{G}$, 
\begin{align*}
     0 &\ge \hat{\mathsf{L}}(\hat{g}) - \hat{\mathsf{L}}(\tilde{g}) \\
     &= \Delta(\hat{g}, \tilde{g}) + \mathsf{L}(\hat{g}) - \mathsf{L}(\tilde{g})
\end{align*} Plugging in the lower bound in \cref{prop:lb} and the upper bound of $\Delta(\hat{g}, \tilde{g})$ in \cref{prop:instance-dependent}, we obtain
\begin{align*}
    \frac{1}{4e^{c_3}} \|\hat{g} - \tilde{g}\|_2^2 - 4 c_1^2 e^{c_3} \|\tilde{g} - \log p_0\|_2^2 \le \mathsf{L}(\hat{g}) - \mathsf{L}(\tilde{g}) &\le |\Delta(\hat{g}, \tilde{g})| \\
    &\le C \cdot c_3 \left(\|g - \tilde{g}\|_2 \delta_{n, t} + \delta_{n, t}^2 \right),
\end{align*} that is,
\begin{align*}
    \|\hat{g} - \tilde{g}\|_2^2 \le \tilde{C} \left[ e^{2c_3} c_1^2 \|\tilde{g} - \log p_0\|_2^2 + e^{c_3+\log(c_3)} (\delta_{n,t}^2 + \delta_{n, t} \|\hat{g} - \tilde{g}\|_2)\right]
\end{align*}

We pick $\tilde{g}$ such that
\begin{align*}
    \|\tilde{g} - \log p_0\|_2 \le \inf_{g\in \mathcal{G}} \|g - \log p_0\|_2 + \frac{1}{n}
\end{align*} Substituting back into the previous inequality, we obtain
\begin{align*}
\|\hat{g} - \tilde{g}\|_2^2 &\le \tilde{C} \left[ e^{2c_3} c_1^2 \inf_{g\in \mathcal{G}} \|g - \log p_0\|_2^2 + \frac{1}{n^2} + e^{c_3+\log(c_3)} (\delta_{n,t}^2 + \delta_{n, t} \|\hat{g} - \tilde{g}\|_2)\right] \\
&\le \tilde{C} \left[ e^{2c_3} c_1^2 \inf_{g\in \mathcal{G}} \|g - \log p_0\|_2^2 \right] + \tilde{C} \left(2 + 2\tilde{C} \right) e^{2[c_3+\log(c_3)]} \delta_{n,t}^2 + \frac{1}{2} \|\hat{g} - \tilde{g}\|_2^2
\end{align*} by the relation $\delta_{n,t} \ge 1/(n^2)$. The relation $c_3 \ge \log(c_3) \lor \log(c_1) \lor 1$ further yields
\begin{align*}
    \|\hat{g} - \tilde{g}\|_2 \lesssim e^{2c_3} \left(\inf_{g\in \mathcal{G}} \|g - \log p_0\|_2 + \delta_{n, t} \right).
\end{align*} Applying the triangle inequality, we obtain
\begin{align*}
    \|\hat{g} - \log p_0 \|_2 \le \|\hat{g} - \tilde g \|_2 + \|\tilde{g} - \log p_0\|_2 \lesssim e^{2c_3} \left(\inf_{g\in \mathcal{G}} \|g - \log p_0\|_2 + \delta_{n, t} \right).
\end{align*} Finally, observe that for any $x, y$
\begin{align*}
    \hat{p}(y|x) - p_0(y|x) = \exp(\bar{p}(y|x)) \left\{\hat{g}(y, x) - \log p_0(y|x)\right\}
\end{align*} with $\exp(\bar{p}(y|x)) \le e^{c_3}$, hence we have
\begin{align*}
    \|\hat{p} - p_0 \|_2 &= \left\|\exp(\bar{p}(y|x)) \left\{\hat{g}(y, x) - \log p_0(y|x)\right\} \right\|_2 \le \left\|\exp(\bar{p}(y|x))\right\|_\infty \left\|\hat{g}(y, x) - \log p_0(y|x)\right\|_2 \\
    &\le e^{3c_3} \left[\inf_{g\in \mathcal{G}} \|g - \log p_0\|_2 + \delta_n + \sqrt{\frac{t + \log (nc_3)}{n}}\right],
\end{align*} this concludes the proof.

\end{proof}

\subsection{Proof of Corollary \ref{cor:main}}

For the TV-distance, by using the inequality $(\mathbb{E}[X])^2 \le \mathbb{E}[X^2]$, we obtain
\begin{align*}
    \mathsf{R}_{\mathsf{TV}}(p_0, \hat{p}) &= \int \int_{\mathcal{Y}} \left|\hat{p}_0(y|x) - p_0(y|x) \right|dy \mu_{0,x}(dx) \\
    &\le \sqrt{\int \int_{\mathcal{Y}} \left|\hat{p}_0(y|x) - p_0(y|x) \right|^2 dy \mu_{0,x}(dx)} = \|\hat{p} - p_0\|_2 \le \delta_{\rm stat} \,,
\end{align*}
where the last conclusion follows from Theorem \ref{thm:main}. 
\\\\
\noindent
As for the general $f$-divergence, we perform a Taylor expansion: 
\begin{align*}
    & \mathsf{R}_{\mathsf{D}_f}(p_0, \hat{p}) \\
    & = \int f\left(\frac{\hat p(y\mid x)}{p_0(y \mid x)}\right) \ p_0(dy \mid x) \mu_{0, x}(dx) \\
    & = \int f(1) \ p_0(dy \mid x) \mu_{0, x}(dx)  + \int \left(\frac{\hat p(y\mid x)}{p_0(y \mid x)} - 1\right)f'(1)\ p_0(dy \mid x) \mu_{0, x}(dx) \\
    & \qquad \qquad \qquad + \int \left(\frac{\hat p(y\mid x)}{p_0(y \mid x)} - 1\right)^2f''\left(\lambda + (1-\lambda)\frac{\hat p(y\mid x)}{p_0(y \mid x)}\right) \ p_0(dy \mid x) \mu_{0, x}(dx) \\
    & = f'(1) \left(\int \hat p(dy \mid x) - 1\right) + \int \left(\frac{\hat p(y\mid x)}{p_0(y \mid x)} - 1\right)^2f''\left(\lambda + (1-\lambda)\frac{\hat p(y\mid x)}{p_0(y \mid x)}\right) \ p_0(dy \mid x) \mu_{0, x}(dx) 
\end{align*}
where in the last equality we have used the fact that $f(1) = 0$. Now consider the second term; as $\hat p $ is upper bounded by the choice of our estimator and $p_0$ is lower bounded by Condition \ref{cond:regularity}, we can upper bound the double derivative of $f$ by some constant. Therefore, we have: 
\begin{align*}
    & \int_\cX \int_\cY \left(\frac{\hat p(y\mid x)}{p_0(y \mid x)} - 1\right)^2f''\left(\lambda + (1-\lambda)\frac{\hat p(y\mid x)}{p_0(y \mid x)}\right) \ p_0(dy \mid x) \mu_{0, x}(dx) \\
    & \le C \int_\cX \int_\cY \left(\frac{\hat p(y\mid x)}{p_0(y \mid x)} - 1\right)^2 \ p_0(dy \mid x) \mu_{0, x}(dx) \\
    & \le C \int_\cX \int_\cY \frac{(\hat p(y \mid x) - p_0(y \mid x))^2}{p_0(y \mid x)} \ dy \ \mu_{0, x}(dx) \\
    & \le Cc_1 \int_\cX \int_\cY (\hat p(y \mid x) - p_0(y \mid x))^2 \ dy \ \mu_{0, x}(dx) \\
    & \le Cc_1 \|\hat p - p_0\|_2^2 \le Cc_1 \delta_{\rm stat} \,.
\end{align*}
Now, with respect to the first sum if $\hat p = \hat p_{\rm norm}$, then it is $0$, and consequently, we have: 
$$
\mathsf{R}_{\mathsf{D}_f}(p_0, \hat{p}) \lesssim \delta_{\rm stat}. 
$$
On the other hand, if it is not normalized, then we have: 
\begin{align*}
\int_\cX \left(\int_{\cY} \left(\hat p(y \mid x)  - 1\right) \ dy \ \mu_{0, x} \ dx\right) & = \int_\cX \left(\int_{\cY} \left(\hat p(y \mid x)  - p_0(y \mid x)\right) \ dy \ \mu_{0, x} \ dx\right) \\
& \le \int_\cX \left(\int_{\cY} \left|\hat p(y \mid x)  - p_0(y \mid x)\right| \ dy \ \mu_{0, x} \ dx\right)  \\
& \le \|\hat p - p_0\|_2 \lesssim \delta_{\rm stat}
\end{align*}
As a consequence, we have: 
$$
\mathsf{R}_{\mathsf{D}_f}(p_0, \hat{p}) \lesssim \sqrt{\delta_{\rm stat}}. 
$$

\subsection{Proof of Corollary \ref{cor:lowdim_factor}}
The proof of this corollary follows from the proof of Corollary \ref{cor:lowdim_compose} by observing the fact that $f_0(x, y) = \log{p_0(y \mid x)}$ belongs to $\mathcal{H}_{\sf hcm}(d, 2, \mathcal{O}, C)$ and consequently $p_0(y \mid x)$ belongs to $\mathcal{H}_{\sf hcm}(d, 3, \mathcal{O}, C)$.

\subsection{Proof of Corollary \ref{cor:lowdim_compose}}
The proof of this corollary essentially follows from the proof of Theorem \ref{thm:main}. Suppose $p_0(y \mid x) \in \mathcal{H}_{\sf hcm}(d, l, \mathcal{O}, C)$. Then it is immediate that $f_0(x, y) = \log{p_0(y \mid x)} \in \mathcal{H}_{\sf hcm}(d, l+1, \mathcal{O}, C)$ (as we are composing the a smooth function $\log{}$ with the conditional density. 
Then by Theorem 4 of \cite{fan2024factor} we know that: 
$$
\inf_{g \in \cG} \|g - \log{p_0}\|_2^2 \le c_5 N^{-4\gamma_*}
$$
where $\cG = \mathcal{H}_{\sf nn}(d_x + d_y, c_1, N, C_3), \gamma_* = \min_{(\beta, d) \in \mathcal{0}}(\beta/d)$. 
In other words, the approximation error above is achieved by a collection of neural networks with constant depth $c_1 = L$ and width $N$. 
As we are using a constant depth $c_1 = L$, the bound on Theorem \ref{thm:main} implies: 
$$
\|\hat p - p_0\|_2^2 \lesssim N^{-4\gamma_*} + \frac{(N^2 + c')\log{n}}{n}
$$
with probability $\ge 1 - 2n^{-c'}$ (for example, one may take $c' = 100$ as mentioned in the statement of the Corollary in the main draft). Now balancing the bias and the variance, and choosing $N \asymp (n/\log{n})^{1/(2 + 4\gamma_*)}$ we obtain that with probability $\ge 1 - 2n^{-c'}$, we have: 
$$
\|\hat p - p_0\|_2^2 \lesssim \left(\frac{n}{\log{n}}\right)^{\frac{2\gamma_*}{2\gamma_* + 1}} \,.
$$
Hence, the bound on $\mathsf{R}_{\sf TV}(\hat p, p_0)$ and $\mathsf{R}_{\sf D_f}(p_0, \hat p)$ follows from Corollary \ref{cor:main} and the bound on $\mathsf{R}_{\sf TV}(p_{\hat Y \mid X}, p_0)$ and $\sqrt{\mathsf{R}_{\sf KL}(p_0, p_{\hat Y \mid X})}$ follows from Theorem \ref{thm:main2}.

\subsection{Proof of Proposition \ref{prop:lb}}
\label{sec:proof-prop-lb}

    Denote $F(u, v) = -u \log \{\sigma(v)\} - \log\{1 - \sigma(v)\}$, it follows from second-order Tayler expansion that
    \begin{align*}
        F(u, v) - F(u, \tilde{v}) = \frac{\partial F}{\partial v} (u,\tilde{v}) \cdot (v - \tilde{v})+ \frac{1}{2} \frac{\partial^2 F}{\partial v^2} (u, \bar{v}) \cdot (v - \tilde{v})^2
    \end{align*} where $\bar{v} = w v + (1-w) v$ for some $w\in [0,1]$. It follows from basic calculation and the definition of $\sigma(\cdot)$ that
    \begin{align*}
        \frac{\partial F}{\partial v} (u, v) &= -u (1 - \sigma(v)) + \sigma(v) \\
        &= (1 + u) \left[\sigma(v) - \sigma(\log(u))\right] \\
        &= (1 + u) \sigma'(m(\log u, v)) (v - \log(u))
    \end{align*} where $m(u, v) = \omega u + (1-\omega) v$, $\sigma'(t) = \sigma(t) (1 - \sigma(t))$, and
    \begin{align*}
        \frac{\partial^2 F}{\partial v^2}(u, v) = (1 + u) \sigma'(v).
    \end{align*} Applying the above second-order expansion with $u = p_0(y|x)$ and $v = g(y,x)$ and $\tilde{v} = \tilde{g}(y, x)$, we obtain
    \begin{align*}
        \mathsf{L}(g) - \mathsf{L}(\tilde{g}) &= \mathbb{E}_{(X,Y)\sim \nu_0}\left[F(u, v) - F(u, \tilde{v})\right] \\ 
        &= \mathbb{E}_{(X,Y)\sim \nu_0} \left[(1 + p_0(Y|X)) \sigma'(m(\log p_0, \tilde{g})) (\tilde{g}(Y, X) - \log p_0(Y|X)) \{g(Y, X) - \tilde{g}(Y,X)\} \right] \\
        &\qquad \qquad + \mathbb{E}_{(X,Y)\sim \nu_0}\left[(1 + p_0(Y|X)) \sigma'(\bar{g}(Y, X)) (g - \tilde{g})^2(Y, X)\right] \\
        &= \mathsf{T}_1(g, \tilde{g}, \log p_0, m) + \mathsf{T}_2(g, \tilde{g}, \bar{g}).
    \end{align*}
    where $m(\log p(y|x), \tilde{g}(y, x)) = \log p(y|x) \cdot \omega(y, x) + \tilde{g}(y, x) \cdot (1 - \omega(y, x))$ with $\omega(y, x) \in [0,1]$, and $\bar{g}(y, x) = g(y, x) w(y, x) + \tilde{g}(y, x) (1 - w(y, x))$ with $w(y, x) \in [0,1]$. It follows from the Cauchy-Schwarz inequality that
    \begin{align*}
    \mathsf{T}_1(g, \tilde{g}, \log p_0, m) &\le \|g - \tilde{g}\|_2 \cdot \left\|(1 + p_0) \sigma'(m) (\tilde{g} - \log p_0)\right\|_2 \\
    &\le 2c_1 \|g - \tilde{g}\|_2 \|\tilde{g} - \log p_0\|_2
    \end{align*} where the second inequality follows from \cref{cond:regularity} and the uniform bound $\sigma'(t) = \sigma(t) (1 - \sigma(t)) \le 1$. 
    On the other hand, \cref{cond:regularity} and the fact that 
    $$
    \bar{g}(y, x) \in [\min\{{g}(y, x), \tilde{g}(y, x)\}, \max\{{g}(y, x),\tilde{g}(y, x)\}] \,,
    $$
    uniformly for any $(y, x)$, further gives
    \begin{align*}
        \sigma'(\bar{g}(y, x)) = \sigma(\bar{g}(y, x)) \cdot (1 - \sigma(\bar{g}(y, x))) = \frac{e^{\bar{g}(y, x)}}{(1 + e^{\bar{g}(y, x)})^2} \ge \frac{1}{2e^{c_3}},
    \end{align*} together with the non-negativity of $p_0$ further implies
    \begin{align*}
        \mathsf{T}_2(g, \tilde{g}, \bar{g}) &\ge \frac{1}{2e^{c_3}} \|g - \tilde{g}\|_2^2.
    \end{align*} Putting all the pieces together, we can conclude that
    \begin{align*}
        \mathsf{L}(g) - \mathsf{L}(\tilde{g}) &\ge - 2c_1 \|g - \tilde{g}\|_2 \|\tilde{g} - \log p_0\|_2 + \frac{1}{2e^{c_3}} \|g - \tilde{g}\|_2^2 \\
        &\ge \frac{1}{4e^{c_3}} \|g - \tilde{g}\|_2 ^2 - 4 c_1^2 e^{c_3} \|\tilde{g} - \log p_0\|_2^2,
    \end{align*} where the last inequality applies Holder inequality $ab \le \frac{1}{2} a^2 + \frac{1}{2} b^2$ with $a = \frac{1}{\sqrt{2e^{c_3}}} \|g - \tilde{g}\|_2$ and $b = \sqrt{2e^{c_3}} 2 c_1 \|\tilde{g} - \log p_0\|_2$. This completes the proof.

\subsection{Proof of Proposition \ref{prop:instance-dependent}}\label{proof:instance-dependent}

We need the following technical lemma from \cite{gu2024causality}.

\begin{lemma}[Instance-dependent error bound on empirical process, Lemma D.1 in \cite{gu2024causality}]
\label{lemma:ep1}
Suppose the function class $\mathcal{H}$ satisfies $\sup_{h\in \mathcal{H}} \|h\|_\infty \le b$, and for any $\delta \ge \delta_n \ge 1/n$, the local population Rademacher complexity satisfies
\begin{align}
\label{eq:local-rade}
    \mathsf{Rade}_{n, \nu}(\delta;\partial \mathcal{H}) \le b\delta_n \delta
\end{align} and the function $\Phi(h, h', z): \mathcal{H} \times \mathcal{H} \times \mathcal{Z}$ satisfies that, $\nu$-a.s.,
\begin{align}
\label{eq:lipcond}
    \Phi(h, h', Z) = v(h, h', Z) \phi(h - h') ~~~~\text{with}~~~~ |v(h, h', z)|\le L_1, \phi\text{ is } L_2 \text{-Lipschitz and } \phi(0)=0.
\end{align}
Then let $\delta_* = \delta_n + \sqrt{\frac{t+1+\log(nb)}{n}}$
\begin{align*}
    \mathbb{P}\Bigg[\forall h,h'\in \mathcal{H}, ~~ &\Big| \frac{1}{n} \sum_{i=1}^n \Phi(h, h', Z_i) - \mathbb{E}[\Phi(h, h', Z_i)] \Big| \\
    &~~~~~~~~~~\le C(bL_1L_2)\{\delta_* \|h-h'\|_{L_2(\nu)} + \delta_*^2\} \Bigg] \ge 1-e^{-t}.
\end{align*} for some universal constant $C>0$.
\end{lemma}

\begin{proof}[Proof of \cref{prop:instance-dependent}]

\noindent \underline{\sc Step 1. Decomposition of $\Delta(g, \tilde{g})$.} We first decompose $\Delta(g, \tilde{g})$ into several parts. It follows from the definition of the loss \eqref{eq:loss} that
\begin{align*}
    \hat{\mathsf{L}}(g) - \hat{\mathsf{L}}(\tilde{g}) &= \frac{1}{n} \sum_{i=1}^n -\log \{\sigma (\tilde{g}(Y_i, X_i))\} - \left[-\log \{\sigma (g(Y_i, X_i))\}\right]\\
    &\qquad \qquad \frac{1}{n} \sum_{i=1}^n -\log \{1-\sigma (g(\tilde{Y}_i, X_i))\} - \left[-\log \{1-\sigma (\tilde{g}(\tilde{Y}_i, X_i))\}\right] \\
    &= \hat{\mathsf{T}}_1(g, \tilde{g}) + \hat{\mathsf{T}}_2(g, \tilde{g}).
\end{align*} and
\begin{align*}
    {\mathsf{L}}(g) - {\mathsf{L}}(\tilde{g}) &= \mathbb{E}_{(X,Y)\sim \nu_0} \left[ p_0(Y|X) \left\{\log(\sigma(\tilde{g}(Y,X))) - \log (\sigma(g(Y,X))) \right\}\right] \\
    &\qquad \qquad \mathbb{E}_{(X,Y)\sim \nu_0} \left[ \log(1-\sigma(\tilde{g}(Y,X))) - \log (1-\sigma(g(Y,X))) \right] \\
    &= {\mathsf{T}}_1(g, \tilde{g}) + {\mathsf{T}}_2(g, \tilde{g}).
\end{align*} Thus
\begin{align*}
    \Delta(g, \tilde{g}) = \hat{\mathsf{T}}_1(g, \tilde{g}) - {\mathsf{T}}_1(g, \tilde{g}) + \hat{\mathsf{T}}_2(g, \tilde{g}) - {\mathsf{T}}_2(g, \tilde{g}).
\end{align*}

In the rest of the proof, we will derive instance-dependent error bounds on $\hat{\Delta}_k(g, \tilde{g}) = \hat{\mathsf{T}}_k(g,\tilde{g}) - \mathsf{T}_k(g,\tilde{g})$ for $k\in \{2,1\}$ and then put the two pieces together. 

\noindent \underline{\sc Step 2. Error bound of $\hat{\Delta}_2(g, \tilde{g})$. } Let $F(v) = \log (1 - \sigma(v))$. We can write $\hat{\Delta}_2(g, \tilde{g})$ as 
\begin{align*}
    \hat{\Delta}_{2}(g, \tilde{g}) &= \frac{1}{n} \sum_{i=1}^n F(\tilde{g}(\tilde{Y}_i, X_i)) - F({g}(\tilde{Y}_i, X_i)) - \mathbb{E}_{(X,Y)\sim \nu_0}\left[F(\tilde{g}(Y, X)) - F({g}(Y, X))\right]\\
    &= \frac{1}{n} \sum_{i=1}^n F(\tilde{g}(Z_i)) - F({g}(Z_i)) - \mathbb{E}\left[F(\tilde{g}(Z)) - F({g}(Z))\right].
\end{align*} where $Z_i = (X_i, \tilde{Y}_i)$ are i.i.d. samples from $Z = (X, Y) \sim \nu_0$. It follows from the mean-value theorem that for any $v, \tilde{v} \in \mathbb{R}$, 
\begin{align*}
    F(v) - F(\tilde{v}) = F'(\bar{v}) (v - \tilde{v}) = -\sigma(\bar{v}) (v - \tilde{v})
\end{align*} where $\bar{v} = \omega v + (1 - \omega) \tilde{v}$ with $\omega \in [0,1]$. Thus,
\begin{align*}
    \hat{\Delta}_{2}(g, \tilde{g}) &= \frac{1}{n} \sum_{i=1}^n -\sigma(\bar{g}(Z_i)) (g(Z_i) - \tilde{g}(Z_i)) - \mathbb{E}[-\sigma(\bar{g}(Z)) (g(Z) - \tilde{g}(Z))]
\end{align*} with $\bar{g}(Z) = \omega(Z) g(Z) + (1 - \omega(Z)) \tilde{g}(Z)$ with $\omega(Z)$ depending only on $g(Z), \tilde{g}(Z)$.

Now we apply \cref{lemma:ep1} with $\mathcal{H} = \mathcal{G}$, $b = c_3$, $\nu = \nu_0$, $\Phi(h, h', z) = v(h, h', z) \phi(h- h')$, where $v(h, h', z) = -\sigma(\bar{g})$ and $\phi(t) = t$. It is easy to verify that \eqref{eq:local-rade} holds given \cref{cond:func-class-g} (2), and \eqref{eq:lipcond} holds with $L_1 = L_2 = 1$, and
\begin{align*}
\hat{\Delta}_{2}(g, \tilde{g}) = \frac{1}{n} \sum_{i=1}^n \Phi(g, \tilde{g}, Z_i)  - \mathbb{E}[\Phi(g, \tilde{g}, Z)].
\end{align*} Then \cref{lemma:ep1} shows the following event
\begin{align}
    \mathcal{A}_2 = \left\{\forall g, \tilde{g} \in \mathcal{G}, ~~ |\hat{\Delta}_{2}(g, \tilde{g})| \le C \left\{\|g - \tilde{g}\|_2 \delta_{n, t} + \delta_{n, t}^2 \right\}\right\}
\end{align} satisfies $\mathbb{P}\left[\mathcal{A}_2^c\right] \le e^{-t}$.

\noindent \underline{\sc Step 3. Error bound of $\hat{\Delta}_1(g, \tilde{g})$. } Let $F(v) = \log (\sigma(v))$. We can write $\hat{\Delta}_1(g, \tilde{g})$ as 
\begin{align*}
    \hat{\Delta}_{1}(g, \tilde{g}) &= \frac{1}{n} \sum_{i=1}^n F(\tilde{g}({Y}_i, X_i)) - F({g}({Y}_i, X_i)) - \mathbb{E}_{(X,\tilde{Y})\sim \nu_0}\big[ p_0(Y|X) \large\{F(\tilde{g}(Y, X)) - F({g}(Y, X))\large\}\big]\\
    &= \frac{1}{n} \sum_{i=1}^n F(\tilde{g}(Z_i)) - F(\tilde{g}(Z_i)) - \mathbb{E}_{{Z} \sim \mu_0}\left[F(\tilde{g}({Z})) - F({g}({Z}))\right].
\end{align*} 
where $Z_i = (X_i, {Y}_i)$ are i.i.d. samples from $Z = (X, Y) \sim \mu_0 = \mu_{0,x} \cdot p_0(y|x)$. It follows from the mean-value theorem that for any $v, \tilde{v} \in \mathbb{R}$, 
\begin{align*}
    F(v) - F(\tilde{v}) = F'(\bar{v}) (v - \tilde{v}) = [1-\sigma(\bar{v})] (v - \tilde{v})
\end{align*} where $\bar{v} = \omega v + (1 - \omega) \tilde{v}$ with $\omega \in [0,1]$. Thus,
\begin{align*}
    \hat{\Delta}_{1}(g, \tilde{g}) &= \frac{1}{n} \sum_{i=1}^n \{1-\sigma(\bar{g}(Z_i))\} (g(Z_i) - \tilde{g}(Z_i)) - \mathbb{E}[\{1-\sigma(\bar{g}(Z))\} (g(Z) - \tilde{g}(Z))]
\end{align*} with $\bar{g}(Z) = \omega(Z) g(Z) + (1 - \omega(Z)) \tilde{g}(Z)$ with $\omega(Z)$ depending only on $g(Z), \tilde{g}(Z)$.

Now we apply \cref{lemma:ep1} with $\mathcal{H} = \mathcal{G}$, $b = c_3$, $\nu = \mu_0$, $\Phi(h, h', z) = v(h, h', z) \phi(h- h')$, where $v(h, h', z) = 1-\sigma(\bar{g})$ and $\phi(t) = t$. It is easy to verify that \eqref{eq:local-rade} holds given \cref{cond:func-class-g} (2), and \eqref{eq:lipcond} holds with $L_1 = L_2 = 1$, and
\begin{align*}
\hat{\Delta}_{1}(g, \tilde{g}) = \frac{1}{n} \sum_{i=1}^n \Phi(g, \tilde{g}, Z_i)  - \mathbb{E}[\Phi(g, \tilde{g}, Z)].
\end{align*} Then \cref{lemma:ep1} shows the following event
\begin{align}
    \mathcal{A}_1 = \left\{\forall g, \tilde{g} \in \mathcal{G}, ~~ |\hat{\Delta}_{2}(g, \tilde{g})| \le C \left\{\|g - \tilde{g}\|_2 \delta_{n, t} + \delta_{n, t}^2 \right\}\right\}
\end{align} satisfies $\mathbb{P}\left[\mathcal{A}_1^c\right] \le e^{-t}$.

\noindent \underline{\sc Step 4. Conclusion. } Now under $\mathcal{A}_1 \cap \mathcal{A}_2$, which occurs with probability at least
\begin{align*}
    \mathbb{P}\left[\mathcal{A}_1 \cap \mathcal{A}_2\right] = 1 - \mathbb{P}\left[\mathcal{A}_1^c \cup \mathcal{A}_2^c\right] \ge 1 - 2e^{-t},
\end{align*} by union bound, we have, by triangle inequality, that
\begin{align*}
    \forall g, \tilde{g} \in \mathcal{G}, \qquad |\Delta(g, \tilde{g})| \le |\hat\Delta_1(g, \tilde{g})| + |\hat\Delta_2(g, \tilde{g})| \le C (\delta_{n, t} \|g - \tilde{g}\|_2 + \delta_{n, t}^2).
\end{align*}
\end{proof}

\section{Proofs for Implicit Density Estimator}

\subsection{Proof of Theorem \ref{thm:main2}}
In this section, we present the proof of Theorem \ref{thm:main2} assuming Proposition \ref{thm:score_est}. The proof of Proposition \ref{thm:score_est} is delegated to Section \ref{sec:proof_score_est}. For notational simplicity, define $\varepsilon_{\rm score}(x, t)$ as the estimation error of the score function given $X = x$ and at time $t$, i.e.: 
$$
\int (\hat{s}_K(y, t|X) - s^\star(y, t|X) )^2 \ p_t(y|X) \ dy = \varepsilon_{\rm score}(x, t) \,.
$$
Furthermore, define $\varepsilon_{\rm score}(x)$ as: 
$$
\varepsilon_{\rm score}(x) = \sum_{n=0}^{N+1}(t_{n+1} - t_n) \varepsilon_{\rm score}(x, T - t_n) \,. 
$$
This implies: 
$$ 
\sum_{n=0}^{N+1}(t_{n+1} - t_n)\left(\int (\hat{s}_K(y, T - t_n|X) - s^\star(y, T-t_n|X) )^2 \ p_t(y|X) \ dy\right) = \varepsilon_{\rm score}(X) \,.
$$
It is immediate from Proposition \ref{thm:score_est} that $\bbE_X(\varepsilon_{\rm score}(X, t)) \le \delta_{\rm score}(t)$, which further yields: 
\begin{align*}
& \bbE_X\left[\sum_{n=0}^{N+1}(t_{n+1} - t_n)\left(\int (\hat{s}_K(y, T - t_n|X) - s^\star(y, T-t_n|X) )^2 \ p_t(y|X) \ dy\right)\right] \\
& = \bbE_X[\varepsilon_{\rm score}(X)] \\
& \le \sum_{n=0}^{N+1}(t_{n+1} - t_n) \delta_{\rm score}(T - t_n) \,.
\end{align*}
Now let us consider Theorem 1 \cite{benton2023nearly} or Theorem 2 of \cite{huang2013oracle}. An application of any of these theorems yields: 
$$
\KL\left(p_{\delta}(\cdot \mid X) \mid p_{\hat Y \mid X}\right) \le C\left[\varepsilon_{\rm score}(X) + \kappa^2 N d_y + \kappa T d_y + de^{-2T}\right] \,.
$$
It is immediate from Equation \eqref{eq:discrete-timestep} that $\kappa \asymp (T + \log{(1/\delta)})/N$. Taking $T = O(1)$, we get: 
\begin{equation}
\label{eq:kl_bound}
\KL\left(p_{\delta}(\cdot \mid X) \mid p_{\hat Y \mid X}\right) \le C\left[\varepsilon_{\rm score}(X) + \frac{(T + \log{(1/\delta)})^2}{N}d_y + \frac{T(T+ \log{(1/\delta)})}{N} d_y + d_ye^{-2T}\right] \,.
\end{equation}
Note that although the left-hand side is a function of $X$, the last three terms of the bound on the RHS does not depend on $X$ as per Condition \ref{cond:regularity}, the first and the second moment of the conditional distribution of $Y$ given $X$ is uniformly bounded over $X$. Taking expectation with respect to $X$ on bound side of Equation \eqref{eq:kl_bound}, we have: 
\begin{align*}
& \bbE_X\left[\KL\left(p_{\delta}(\cdot \mid X) \mid p_{\hat Y \mid X}\right)\right] \\
& \le C\left[\sum_{n=0}^{N+1}(t_{n+1} - t_n) \delta_{\rm score}(T - t_n) + \frac{(T + \log{(1/\delta)})^2}{N}d_y + \frac{T(T+ \log{(1/\delta)})}{N} d_y + d_ye^{-2T}\right] \,.
\end{align*}

\subsection{Proof of Proposition \ref{thm:score_est}}
\label{sec:proof_score_est}
Set $\Omega = \mathbb{R}^d$ and $\Theta = [0,1]^d$. We further define: 
\begin{align*}
    D(x, t) &= \int_{\Theta} \frac{1}{(\sqrt{2\pi}\sigma_t)^d} \exp\left(-\frac{\|x - m_t y\|_2^2}{2\sigma_t^2}\right)p(y) dy = p_t(x) \in \mathbb{R}\\
    N(x, t) &= -\int_{\Theta} \frac{x - m_t y}{\sigma_t}\frac{1}{(\sqrt{2\pi} \sigma_t)^d} \exp\left(-\frac{\|x - m_t y\|_2^2}{2\sigma_t^2}\right) p(y) dy \in \mathbb{R}^d
\end{align*} 
It is immediate that the score function $s(x, t)$ of $Y_t$ (forward process) satisfies 
$$
s(x, t) = \frac{1}{\sigma_t} N(x, t)/D(x, t) \,.
$$
Let $\hat{N}, \hat{D}$ and be the estimated counterparts of $(N, D)$, with $p_0$ replaced by $\hat p_0$, i.e.
\begin{align*}
    \hat D(x, t) &= \int_{\Theta} \frac{1}{(\sqrt{2\pi}\sigma_t)^d} \exp\left(-\frac{\|x - m_t y\|_2^2}{2\sigma_t^2}\right)\hat p_0(y \mid X = x) dy = p_t(x) \in \mathbb{R}\\
    \hat N(x, t) &= -\int_{\Theta} \frac{x - m_t y}{\sigma_t}\frac{1}{(\sqrt{2\pi} \sigma_t)^d} \exp\left(-\frac{\|x - m_t y\|_2^2}{2\sigma_t^2}\right) \hat p_0(y \mid X = x) dy \in \mathbb{R}^d
\end{align*}
Last but not least, let $\hat N^{\rm emp}, \hat D^{\rm emp}$, denote the empirical counterpart of $(N, D)$, where we replace the population average in the definition of $(\hat N, \hat D)$ by sample average: 
\begin{align*}
\hat D^{\rm emp}(x, t) &= \frac1K \sum_{i = 1}^K \hat p_0\left(\frac{x - \sigma_t Z_i}{m_t}\right) \\
\hat N^{\rm emp}(x, t) & = \frac{1}{K}\sum_{i = 1}^K Z_i \hat p_0\left(\frac{x - \sigma_t Z_i}{m_t}\right)
\end{align*}
where $Z_1, \dots, Z_n \overset{i.i.d.}{\sim} \cN(0, 1)$. Recall that $\sigma_t \hat s(x, t) = \hat N^{\rm emp}(x, t)/\hat D^{\rm emp}(x, t)$. For ease of presentation, define $\sigma_t \tilde s(x, t) = \hat N(x, t)/\hat D(x, t)$. 
An application of the inequality $(a + b)^2 \le 2(a^2 + b^2)$ yields: 
\begin{equation}
\label{eq:score_break_1}
\bbE_{Y_t}\left[(\hat{s}(Y_t, t) - s(Y_t, t) )^2 \right] \le 2\underbrace{\bbE_{Y_t}\left[(\hat{s}(Y_t, t) - \tilde s(Y_t, t) )^2 \right]}_{:= T_1} + 2\underbrace{\bbE_{Y_t}\left[(\tilde{s}(Y_t, t) - s(Y_t, t) )^2 \right]}_{:=T_2}
\end{equation}
We would like to highlight that both $T_1$ and $T_2$ are random variables; the randomness of $T_2$ stems from the observed data $\cD_n$ (through $\hat p_0(y \mid X = x)$) and the randomness of $T_1$ arises both from $\cD_n$ and the randomness of $\{Z_1, \dots, Z_n\}$. 
We start with $T_2$: 
\begin{align*}
& \bbE_{Y_t}\left[(\tilde{s}(Y_t, t) - s(Y_t, t) )^2 \right] \\
= & \int_{\Omega} (\tilde{s}(x, t) - s(x, t) )^2 p_t(x) dx \\
= & \int_{\Omega} (\tilde{s}(x, t) - s(x, t) )^2 D(x, t) dx \hspace{.1in} [\because D(x, t) = p_t(x)] \\
\le & \frac{2}{\sigma_t^2} \left(\int_{\Omega} \left\|\frac{\hat{N}(x,t)}{\hat{D}(x,t)}\right\|_2^2 \frac{(\hat{D}(x,t) - D(x,t))^2}{D(x,t)} dx + \int_\Omega \frac{\|\hat{N}(x,t) - N(x,t)\|^2}{D(x,t)} dx\right) 
\end{align*}
Let $\hat{\Delta}(y) = \hat{p}_{0}(y \mid X = x) - p_0(y \mid X = x)$. For the second summand, it follows from the Cauchy-Schwarz inequality that
\allowdisplaybreaks
\begin{align*}
& \int_\Omega \frac{\|\hat{N}(x,t) - N(x,t)\|^2}{D(x,t)} dx \\
&= \int_{\Omega} \frac{1}{D(x, t)} \left\|\int_{\Theta} \frac{x - m_t y}{\sigma_t}\frac{1}{(\sqrt{2\pi} \sigma_t)^d} \exp\left(-\frac{\|x - m_t y\|_2^2}{2\sigma_t^2}\right) \hat{\Delta}(y) dy \right\|_2^2 \\
&\le \|p^{-1}\|_\infty \int_{\Omega} \frac{1}{D(x, t)} \left\|\int_{\Theta} \frac{x - m_t y}{\sigma_t}\frac{1}{(\sqrt{2\pi} \sigma_t)^d} \exp\left(-\frac{\|x - m_t y\|_2^2}{2\sigma_t^2}\right) \sqrt{p(y)} \cdot \hat{\Delta}(y) dy \right\|_2^2 \\
&\le \|p^{-1}\|_\infty \int_{\Omega} \left(\frac{D(x, t)}{D(x, t)} \int_\Theta \hat{\Delta}^2(y) \left\|\frac{x - m_t y}{\sigma_t}\right\|_2^2   \frac{1}{(\sqrt{2\pi} \sigma_t)^d} \exp\left(-\frac{\|x - m_t y\|_2^2}{2\sigma_t^2}\right) dy \right) dx \\
&= \|p^{-1}\|_\infty \int_\Theta \hat{\Delta}^2(y) \left(\int \left\|\frac{x - m_t y}{\sigma_t}\right\|_2^2   \frac{1}{(\sqrt{2\pi} \sigma_t)^d} \exp\left(-\frac{\|x - m_t y\|_2^2}{2\sigma_t^2}\right) dx \right) dy \\
&= \|p^{-1}\|_\infty \int_\Theta \hat{\Delta}^2(y) \left(\int \left\|z\right\|_2^2   \frac{1}{(\sqrt{2\pi})^d} \exp\left(-\|z\|_2^2/2\right) dz \right) dy \\
&= \|p^{-1}\|_\infty d \int_\Theta \hat{\Delta}^2(y) dy.
\end{align*}
Observe that it also follows from Cauchy Schwarz inequality that, similarly, 
\begin{align*}
    \|\hat{N}(x, t)\|_2^2 &\le \hat{D}^2(x, t) \left(\int_\Theta \left\|\frac{x - m_t y}{\sigma_t}\right\|_2^2   \frac{1}{(\sqrt{2\pi} \sigma_t)^d} \exp\left(-\frac{\|x - m_t y\|_2^2}{2\sigma_t^2}\right) \hat{p}(y)dy \right) \\
    &\le \hat{D}^2(x, t) \int \|z\|_2^2 \frac{1}{(\sqrt{2\pi})^d} \exp\left(-\|z\|_2^2/2\right) \hat{p}\left(\frac{x + \sigma_t z}{m_t}\right) dz \\
    &\le \hat{D}^2(x, t) d \|\hat{p}\|_\infty.
\end{align*}
Turning to the first summand, plug-in our uniform bound above, 
\begin{align*}
\int_{\Omega} \left\|\frac{\hat{N}(x,t)}{\hat{D}(x,t)}\right\|_2^2 \frac{(\hat{D}(x,t) - D(x,t))^2}{D(x,t)} dx & \le d \|\hat{p}\|_\infty \int_{\Omega} \frac{(\hat{D}(x,t) - D(x,t))^2}{D(x,t)} dx \\
&\le  d \|\hat{p}\|_\infty \|p^{-1}\|_\infty \int_\Theta \hat{\Delta}^2(y) dy.
\end{align*}
Putting both the bounds together, we obtain
\begin{align*}
    T_2 := \bbE_{Y_t}\left[(\tilde{s}(Y_t, t) - s(Y_t, t) )^2 \right]  \le \frac{d \|p^{-1}\|_\infty \left(1 + \|\hat{p}\|_\infty\right)}{\sigma_t^2} \int_\Theta \hat{\Delta}^2(y) dy.
\end{align*}
Now for $T_1$ of Equation \eqref{eq:score_break_1}, we have: 
\begin{align*}
   &  \int_{\Omega} \left\|\hat{s}(x, t) - \tilde s(x, t)\right\|^2 p_t(x) dx \\
   & = \frac{1}{\sigma^2_t}\int_\Omega \left\|\frac{\hat N^{\rm emp}(x, t)}{\hat D^{\rm emp}(x, t)} - \frac{\hat N(x, t)}{\hat D(x, t)}\right\|_2^2 \ p_t(x) \ dx \\
   & = \sigma^2_t\int_\Omega \left\|\frac{\hat N(x, t)}{\hat D(x, t)} - \frac{\hat N(x, t)}{\hat D^{\rm emp}(x, t)} + \frac{\hat N(x, t)}{\hat D^{\rm emp}(x, t)} - \frac{\hat N^{\rm emp}(x, t)}{\hat D^{\rm emp}(x, t)}\right\|_2^2 \ p_t(x) \ dx \\
   & = \sigma^2_t\int_\Omega \left\|\frac{\hat N(x, t)}{\hat D(x, t)}\left(1 - \frac{\hat D(x, t)}{\hat D^{\rm emp}(x, t)}\right) + \frac{1}{\hat D^{\rm emp}(x, t)}\left(\hat N(x, t) - \hat N^{\rm emp}(x, t)\right)\right\|_2^2 \ p_t(x) \ dx \\
   & \le \frac{2}{\sigma^2_t}\int_\Omega \left\|\frac{\hat N(x, t)}{\hat D(x, t)}\left(1 - \frac{\hat D(x, t)}{\hat D^{\rm emp}(x, t)}\right)\right\|^2_2 \ p_t(x) \ dx \\
   & \qquad \qquad \qquad + \frac{2}{\sigma^2_t}\int_\Omega \left\|\frac{1}{\hat D^{\rm emp}(x, t)}\left(\hat N(x, t) - \hat N^{\rm emp}(x, t)\right)\right\|_2^2 \ p_t(x) \ dx \\
   & \le  \frac{2}{\sigma^2_t}\int_\Omega \left\|\frac{\hat N(x, t)}{\hat D(x, t)}\right\|_2^2 \frac{(\hat D(x, t) - \hat D^{\rm emp}(x, t))^2}{(\hat D^{\rm emp}(x, t))^2} \ p_t(x) dx \\
   & \qquad \qquad \qquad + \frac{2}{\sigma^2_t}\int_\Omega \frac{\|\hat N(x, t)- \hat N^{\rm emp}(x, t)\|_2^2}{(\hat D^{\rm emp}(x, t))^2} \ p_t(x) dx \\
   & \le \frac{2d\|\hat p_0\|_{\infty}}{\sigma^2_t} \int_\Omega \frac{(\hat D(x, t) - \hat D^{\rm emp}(x, t))^2}{(\hat D^{\rm emp}(x, t))^2} \ p_t(x) dx + \frac{2}{\sigma^2_t}\int_\Omega \frac{\|\hat N(x, t)- \hat N^{\rm emp}(x, t)\|_2^2}{(\hat D^{\rm emp}(x, t))^2} \ p_t(x) dx \\
   & \le \frac{2d\|\hat p_0\|_{\infty}}{\sigma^2_t} \int_\Omega(\hat D(x, t) - \hat D^{\rm emp}(x, t))^2 \ p_t(x) dx \\
   & \qquad \qquad \qquad + \frac{2\|\hat p_0^{-1}\|_\infty^2}{\sigma^2_t} \int_\Omega \|\hat N(x, t)- \hat N^{\rm emp}(x, t)\|_2^2 \ p_t(x) \ dx 
\end{align*}
Therefore, we need to provide an upper bound of 
$$
\int_\Omega(\hat D(x, t) - \hat D^{\rm emp}(x, t))^2 \ p_t(x) dx \qquad \text{and} \qquad  \int_\Omega \|\hat N(x, t)- \hat N^{\rm emp}(x, t)\|_2^2 \ p_t(x) \ dx \,.
$$
For notational simplicity, define $f(Z; x)$ as: 
$$
f(Z; x) = Z\hat p_0\left(\frac{x - \sigma_t Z}{m_t}\right) \,.
$$
Then, we have: 
$$
\hat N(x, t) - \hat N^{\rm emp}(x, t) = (\bbP_n - \bbP)f(Z; x)
$$
where $\bbP_n$ (resp. $\bbP$) denotes the empirical measure (resp. population measure) with respect to $Z$. Using this notation, we have: 
\begin{align*}
    &  \int_\Omega \|\hat N(x, t)- \hat N^{\rm emp}(x, t)\|_2^2 \ p_t(x) \ dx \\
    &  = \int_\Omega \|(\bbP_n - \bbP)f(Z; x)\|_2^2 \ p_t(x) \ dx \\
    & = \frac{1}{K^2} \sum_{i = 1}^K \int_\Omega \|f(Z_i; x) - \bbE_Z[f(Z; x)]\|^2 \ dx  \\
    & \qquad \qquad + \frac{1}{K^2}\sum_{i \neq j}\int_\Omega (f(Z_i; x) - \bbE_Z[f(Z; x)])^\top(f(Z_j; x) - \bbE_Z[f(Z; x)]) \ dx \\
    & := S_1 + S_2 \,.
\end{align*}
The term $S_2$ can be written as a sum of degenerate $U$-statistics. Towards that goal, define: 
\begin{align*}
 K(x, y) & =  \int_\Omega \hat p_0\left(\frac{u  - \sigma_t x}{m_t}\right)\hat p_0\left(\frac{u - \sigma_t y}{m_t}\right) \ p_t(u) \ du \\
    h(x, y) & = x^\top y \ K(x, y)\\
    h^D(Z, Z') & = h(Z, Z') - \bbE[h(Z, Z') \mid Z] - \bbE[h(Z, Z') \mid Z'] + \bbE[h(Z, Z')] \,.
\end{align*}
Then it is immediate that: 
$$
h^D(Z_i, Z_j) = \int_\Omega (f(Z_i; x) - \bbE_Z[f(Z; x)])^\top(f(Z_j; x) - \bbE_Z[f(Z; x)]) \ dx \,.
$$
As a consequence, we have: 
$$
S_2 = \frac{1}{n^2} \sum_{i \neq j} h^D(Z_i, Z_j) \,.
$$
Our next goal is to present a high probability upper bound on $S_2$. However, we need to use a truncation-based argument as $h_D$ is unbounded. Therefore, we divide the summand into two parts: 
\begin{align*}
\frac{1}{n^2} \sum_{i \neq j} h^D(Z_i, Z_j) & = \frac{1}{n^2} \sum_{i \neq j} h^D(Z_i, Z_j)\mathds{1}_{\|Z_i\| \le C\sqrt{d\log{n}}, \|Z_j\| \le C\sqrt{d\log{n}}} \\
& \qquad \qquad + \frac{1}{n^2} \sum_{i \neq j} h^D(Z_i, Z_j)\left(1 - \mathds{1}_{\|Z_i\| \le \sqrt{Cd\log{n}}, \|Z_j\| \le \sqrt{Cd\log{n}}}\right) \\
& \triangleq \frac{1}{n^2} \sum_{i \neq j} h^{n, D}(Z_i, Z_j) + \frac{1}{n^2} \sum_{i \neq j} (h - h^{n, D})(Z_i, Z_j) \,.
\end{align*}
Here, $C$ is a large constant to be defined later. To tackle the first, we first relate the degenerate U-statistics to uncoupled U-statistics using Theorem 1 of \cite{de1995decoupling}, which states that there exists some universal constant $C_2$
$$
\bbP\left(\left|\frac{1}{n^2} \sum_{i \neq j} h^{n, D}(Z_i, Z_j)\right| > t\right) \le C_2 \bbP\left(\left|\frac{1}{n^2} \sum_{i \neq j} h^{n, D}(Z_i, Z'_j)\right| > t/C_2\right)
$$
where $Z_1', \dots, Z_n'$ are i.i.d. copies of $(Z_1, \dots, Z_n)$. Therefore, it is enough to provide an upper bound on the right hand side of the above inequality. Towards that goal, we use Corollary 3.4 of \cite{gine2000exponential}. Note that by the trunction, we have: 
\begin{align*}
    \|h^{n, D}\|_\infty & \le 4Cd\|\hat p_0\|^2_\infty\log{n} \,, \\
    \bbE[(h^{n, D}(Z, Z'))^2] & \le (4Cd\|\hat p_0\|^2_\infty\log{n})^2 \,, \\
    \sup_{z}\bbE\left[(h^{n, D}(Z, Z'))^2 \mid Z = z\right] & \le (4Cd\|\hat p_0\|^2_\infty\log{n})^2 \,, \\
    \sup_z\bbE\left[(h^{n, D}(Z, Z'))^2 \mid Z' = z\right] & \le (4Cd\|\hat p_0\|^2_\infty\log{n})^2  \,.
\end{align*}
Furthermore, define $\|h^{n, D}\|_{L_2 \to L_2}$ as: 
$$
\sup\left\{\bbE[h^{n, D}(Z, Z')f(Z)g(Z')]: \ \bbE[f^2(Z)] \le 1, \bbE[g^2(Z')] \le 1\right\} \,.
$$
Now, for any $(f, g)$ with $\|f\|_2 \le 1, \|g\|_2 \le 1$, we have: 
$$
\bbE[h^{n, D}(Z, Z')f(Z)g(Z')] \le \left(4Cd\|\hat p_0\|^2_\infty \log{n}\right)\bbE[|f(Z)||g(Z')|] \le 4Cd\|\hat p_0\|^2_\infty \log{n} \,.
$$
An application of Corollary 3.4 of \cite{gine2000exponential} yields:  
\begin{align*}
    & \bbP\left(\left|\frac{1}{n^2} \sum_{i \neq j} h^{n, D}(Z_i, Z'_j)\right| > \frac{(4Cd\|\hat p_0\|^2_\infty \log{n})^2}{n}\right) \\
    & \le K_1\exp\left(-K_2\min\left\{(\log{n})^2, \log{n}, n^{1/3}(\log{n})^{2/3}, \sqrt{n\log{n}}\right\}\right) = K_1\exp\left(-K_2 \log{n}\right) \,.
\end{align*}
Here one can make $K_2$ large by choosing large $C$. 
Therefore, we conclude that: 
$$
\frac{1}{n^2} \sum_{i \neq j} h^{n, D}(Z_i, Z'_j) \le \frac{(4Cd\|\hat p_0\|^2_\infty \log{n})^2}{n} \ \text{ with probability } \ \ge 1 - K_1\exp\left(-K_2 \log{n}\right)  \,.
$$
Now, for the other part, we use the tail bound for the norm of a Gaussian random variable. From Example 2.12 of \cite{boucheron2003concentration}, we have for $t \ge d$, 
$$
\bbP(\|Z\|_2^2 \ge t) \le \exp\left(-\frac{t^2}{8}\right) \,.
$$
Therefore, 
\begin{align*}
    \bbE\left[\|Z\|\mathds{1}_{\|Z\| \ge \sqrt{Cd\log{n}}}\right] \le 2\sqrt{Cd\log{n}} \ \exp\left(-\frac{Cd\log{n}}{8}\right) \le \frac{2\sqrt{Cd\log{n}}}{n^{\frac{Cd}{8}}} \,. 
\end{align*}
\begin{align*}
    \bbE\left[\|Z\|\mathds{1}_{\|Z\| \ge \sqrt{Cd\log{n}}}\right] & \le 2\sqrt{Cd\log{n}} \ \exp\left(-\frac{Cd\log{n}}{8}\right) \\
    & = 2\exp\left(\frac{1}{2}\log{(Cd\log{n})} - \frac{Cd\log{n}}{8}\right) \\
    & \le 2\exp\left(- \frac{Cd\log{n}}{9}\right) \hspace{.1in} [\forall \ \text{large } n] \,.
\end{align*}
This immediately implies: 
\begin{align*}
    & \bbE\left[\left|\frac{1}{n^2} \sum_{i \neq j} (h - h^{n, D})(Z_i, Z_j)\right|\right] \\
    & \le \frac{4}{n^2}\sum_{i \neq j} \bbE\left[\|Z_i\|\|Z_j\|K(Z_i, Z_j)\left(1 - \mathds{1}_{\|Z_i\| \le \sqrt{Cd\log{n}}, \|Z_j\| \le \sqrt{Cd\log{n}}}\right)\right] \\
    & \le \frac{4\|\hat p_0\|_\infty^2}{n^2}\sum_{i \neq j} \left(\bbE\left[\|Z_i\|\|Z_j\|\mathds{1}_{\|Z_i\| >\sqrt{Cd\log{n}}}\right] + \bbE\left[\|Z_i\|\|Z_j\|\mathds{1}_{\|Z_j\| > \sqrt{Cd\log{n}}}\right]\right) \\
    & \le \frac{8\sqrt{d}\|\hat p_0\|_\infty^2}{n}\sum_{j=1}^n \bbE\left[\|Z_j\|\mathds{1}_{\|Z_j\| > \sqrt{Cd\log{n}}}\right] \\
    & \le (16\sqrt{d}\|\hat p_0\|_\infty^2)\exp\left(- \frac{Cd\log{n}}{9}\right) \,.
\end{align*}
Therefore, we have: 
$$
\bbP\left(\left|\frac{1}{n^2} \sum_{i \neq j} (h - h^{n, D})(Z_i, Z_j)\right| \ge  (16\sqrt{d}\|\hat p_0\|_\infty^2)\exp\left(- \frac{Cd\log{n}}{18}\right)\right) \le \exp\left(- \frac{Cd\log{n}}{18}\right) \,.
$$
Combining the upper bounds, we obtain, with probability $1 - K_3 \exp(-K_4 \log{n})$: 
\begin{align*}
    S_2 \le \frac{(4Cd\|\hat p_0\|^2_\infty \log{n})^2}{n} + (16\sqrt{d}\|\hat p_0\|_\infty^2)\exp\left(- \frac{Cd\log{n}}{18}\right) \le \frac{32(Cd\|\hat p_0\|^2_\infty \log{n})^2}{n}
\end{align*}
where the second inequality holds for all large $n$, as the first term starts to dominate the second term. Now going back to $S_1$, we have: 
\begin{align*}
    & \frac{1}{K^2} \sum_{i = 1}^K \int_\Omega \|f(Z_i; x) - \bbE_Z[f(Z; x)]\|^2 \ p_t(x) \ dx \\
    & \le \frac{2}{n^2}\sum_{i = 1}^n \int_\Omega \|f(Z_i; x)\|_2^2 \ p_t(x) \ dx + \frac{2}{n}\int_{\Omega} \|\bbE_Z[f(Z; x)]\|_2^2 \ p_t(x) \ dx \\
    & \le 
    \frac{2\|\hat p_0\|_\infty^2}{n}\ \frac{1}{n}\sum_{i = 1}^n \|Z_i\|_2^2 + \frac{2}{n}\int_{\Omega} \bbE_Z[\|f(Z; x)\|_2^2] \ p_t(x) \ dx \\
    & \le  \frac{2\|\hat p_0\|_\infty^2}{n}\ \frac{1}{n}\sum_{i = 1}^n (\|Z_i\|_2^2 - d) + \frac{2d\|\hat p_0\|_\infty^2}{n} + \frac{2d\|\hat p_0\|_\infty^2}{n} 
\end{align*}
Now, by Bernstein's inequality for centered sub-exponential random variables, we have: 
$$
\bbP\left(\frac{1}{\sqrt{n}}\sum_{i = 1}^n (\|Z_i\|_2^2 - d) \ge t\right) \le 2 \exp\left(-\min\left\{\frac{t^2}{B^2d^2}, \frac{t}{Bd}\right\}\right)
$$
for some universal constant $B$. As a consequence, we have: 
$$
\bbP\left(\frac{1}{n}\sum_{i = 1}^n (\|Z_i\|_2^2 - d) \ge \frac{CBd\log{n}}{\sqrt{n}}\right) \le 2 \exp\left(-\min\left\{C^2 \log^2{n}, C\log{n}\right\}\right) = 2\exp\left(-C\log{n}\right) \,.
$$
Therefore, we can conclude that with probability $\ge 1 - 2\exp(-C\log{n})$, we have: 
$$
S_1 \le \frac{2\|\hat p_0\|_\infty^2CBd\log{n}}{n^{3/2}} + \frac{4d\|\hat p_0\|_\infty^2}{n} \le \frac{5Cd\|\hat p_0\|_\infty^2}{n}\,.
$$
where the last inequality holds for all large $n$. Hence, combining the bounds on $S_1$ and $S_2$ we have with probability $\ge 1 - K_1 \exp(-K_2 \log{n})$: 
$$
S_1 + S_2 \le \frac{33(Cd\|\hat p_0\|^2_\infty \log{n})^2}{n} \,.
$$
Next, we turn to the upper bound for the difference between $\hat D^{\rm emp}$ and $\hat D$. For notational simplicity, define: 
$$
g_t(Z; x) = \hat p_0\left(\frac{x - \sigma_t Z}{m_t}\right) \,.
$$
Then we have: 
$$
\hat D^{\rm emp}(x, t) - \hat D(x, t) = (\bbP_n - \bbP)g_t(Z, x)
$$
Therefore, as for the numerator: 
\begin{align*}
    & \int_\Omega(\hat D(x, t) - \hat D^{\rm emp}(x, t))^2 \ p_t(x) dx \\
    & = \int_\Omega((\bbP_n - \bbP)g_t(Z, x))^2 \ p_t(x) dx \\
     & = \frac{1}{K^2} \sum_{i = 1}^K \int_\Omega (g_t(Z_i; x) - \bbE_Z[g_t(Z; x)])^2 \ dx  \\
     & \qquad \qquad + \frac{1}{n^2}\sum_{i \neq j}\int_\Omega (g_t(Z_i; x) - \bbE_Z[g_t(Z; x)])(g_t(Z_j; x) - \bbE_Z[g_t(Z; x)]) \ dx \\
    & := S_3 + S_4 \,.
\end{align*}
Using same argument as for the numerator, we can show that we can express $S_4$ as sum of degenerate $U$-statistics, i.e., 
$$
S_4 = \frac{1}{n^2} \sum_{i \neq j} h^D(Z_i, Z_j), 
$$
where 
\begin{align*}
h^D(Z_1, Z_2) & = g_t(Z_1, x)g_t(Z_2, x) - g_t(Z_1, x)\bbE_{Z_2}[g_t(Z_2, x)] \\
& \qquad - \bbE_{Z_1}(g_t(Z_1, x))g_t(Z_2, x) + \bbE_{Z_1}[g_t(Z_1, x)]\bbE_{Z_2}[g_t(Z_2, x)] \,.
\end{align*}
As $g_t$ is bounded by $\|\hat p_0\|_\infty$, it is immediate that $h^D$ is upper bounded by $4\|\hat p_0\|_\infty^2$. Another application of Corollary 3.4 of \cite{gine2000exponential} yields: 
\begin{align*}
    \bbP\left(\left|\frac{1}{n^2} \sum_{i \neq j} h^D(Z_i, Z_j)\right| \ge 4\|\hat p_0\|_\infty^2x\right) \le K\exp\left(-\frac{1}{K}\min\left\{n^2x^2, nx, nx^{2/3}, nx^{1/2}\right\}\right)
\end{align*}
for some universal constant $K$. Taking $x = (C\log{n})/n$ (for some large enough constant $C$), we have: 
\begin{align*}
      & \bbP\left(\left|S_4\right| \ge \frac{4C\|\hat p_0\|_\infty^2\log{n}}{n}\right) \\
      & =  \bbP\left(\left|\frac{1}{n^2} \sum_{i \neq j} h^D(Z_i, Z_j)\right| \ge \frac{4C\|\hat p_0\|_\infty^2\log{n}}{n}\right) \\
     &  \le K\exp\left(-\frac{1}{K}\min\left\{(C\log{n})^2, (C\log{n}), n^{1/3}(C\log{n})^{2/3}, \sqrt{Cn\log{n}}\right\}\right) \\
     &  \le  K\exp\left(-\frac{C\log{n}}{K}\right) \hspace{.1in} [\forall \ \text{large } n]\,.
\end{align*}
For $S_3$, we use the fact that $\|g\|_\infty \le \|\hat p_0\|_\infty$, which yields $S_3 \le 2\|\hat p_0\|_\infty^2/n$. Hence, we conclude with probability $\ge 1 - K\exp\left(-(C\log{n})/K\right)$: 
$$
\int_\Omega(\hat D(x, t) - \hat D^{\rm emp}(x, t))^2 \ p_t(x) dx \le \frac{2\|\hat p_0\|_\infty^2}{n} + \frac{4C\|\hat p_0\|_\infty^2\log{n}}{n} \le \frac{5C\|\hat p_0\|_\infty^2\log{n}}{n} \,,
$$
where the last inequality holds for all large $n$. Combining the bounds on numerator and denominator, we can finally conclude that with probability $\ge 1 - K\exp\left(-(C\log{n})/K\right)$: 
\begin{align*}
    &  \int_{\Omega} \left\|\hat{s}(x, t) - \tilde s(x, t)\right\|^2 p_t(x) dx \\
    & \le \frac{10Cd\|\hat p_0\|_\infty^3\|\hat p_0^{-1}\|_\infty^2\log{n}}{n\sigma^2_t} + \frac{66(Cd\|\hat p_0\|^2_\infty \|\hat p_0^{-1}\|_\infty\log{n})^2}{n\sigma^2_t} \\
    & \le \frac{KC^2 d^2 (\log{n})^2\|\hat p_0\|_\infty^4\|\hat p^{-1}_0\|_\infty^2}{n\sigma^2_t} \,.
\end{align*}
This completes the proof.

\end{document}